\documentclass{article}

\usepackage{microtype}
\usepackage{graphicx}
\usepackage{url}            
\usepackage{booktabs} 
\usepackage{amsfonts}       
\usepackage{nicefrac}       
\usepackage{microtype}      

\usepackage{amsfonts}
\usepackage{amssymb}
\usepackage{amsmath}
\usepackage{amsthm}

\usepackage{subcaption}
\usepackage{caption}
\usepackage{wrapfig}

\newcommand{\R}{\mathbb{R}}

\newcommand{\abs}[1]{\left\vert#1\right\vert}
\newcommand{\norm}[1]{\left\Arrowvert#1\right\Arrowvert}

\newcommand{\argmin}{\mathop{\text{argmin}}}
\newcommand{\argmax}{\mathop{\text{argmax}}}
\newcommand{\grad}{\nabla}

\newcommand{\defeq}{\overset{\operatorname{def}}{=}}

\newtheorem{theorem}{Theorem}[section]

\newtheorem{prop}[theorem]{Proposition}

\usepackage{hyperref}

\usepackage[accepted]{icml2021}

\icmltitlerunning{Amortized Conditional Normalized Maximum Likelihood}

\begin{document}

\twocolumn[
\icmltitle{Amortized Conditional Normalized Maximum Likelihood: Reliable Out of Distribution Uncertainty Estimation 
}



\icmlsetsymbol{equal}{*}

\begin{icmlauthorlist}
\icmlauthor{Aurick Zhou}{ber}
\icmlauthor{Sergey Levine}{ber}
\end{icmlauthorlist}

\icmlaffiliation{ber}{EECS, University of California, Berkeley, USA}

\icmlcorrespondingauthor{Aurick Zhou}{aurick@berkeley.edu}

\icmlkeywords{Machine Learning, ICML}

\vskip 0.3in
]






\begin{abstract}
While deep neural networks provide good performance for a range of challenging tasks, calibration and uncertainty estimation remain major challenges, especially under distribution shift. 
In this paper, we propose the amortized conditional normalized maximum likelihood (ACNML)
method as a scalable general-purpose approach for uncertainty estimation, calibration, and out-of-distribution robustness with deep networks. 
Our algorithm builds on the conditional normalized maximum likelihood (CNML) coding scheme, which has minimax optimal properties according to the minimum description length principle, but is computationally intractable to evaluate exactly for all but the simplest of model classes. 
We propose to use approximate Bayesian inference technqiues to produce a tractable approximation to the CNML distribution. Our approach can be combined with any approximate inference algorithm that provides tractable posterior densities over model parameters.
We demonstrate that ACNML compares favorably to a number of prior techniques for uncertainty estimation in terms of calibration on out-of-distribution inputs.
\end{abstract}

\section{Introduction}

Current machine learning methods provide unprecedented accuracy across a range of domains, from computer vision to natural language processing. However, in many high-stakes applications, such as medical diagnosis or autonomous driving, rare mistakes can be extremely costly. Thus, effective deployment of learned models requires not only high accuracy, but also a way to measure the certainty in a model's predictions in order to assess risk and allow the model to abstain from making decisions when there is low confidence in the prediction. While deep networks offer excellent prediction accuracy, they generally do not provide the means to accurately quantify their uncertainty. 
This is especially true on out-of-distribution inputs, where deep networks tend to make overconfident incorrect predictions \citep{ovadia2019trust}. 
In this paper, we tackle the problem of obtaining reliable uncertainty estimates under distribution shift, with the aim of producing models that can reliably report their uncertainty even when presented with unexpected inputs.


Most prior work approaches the problem of uncertainty estimation from the standpoint of Bayesian inference. By treating parameters as random variables with some prior distribution, Bayesian inference can compute posterior distributions that capture a notion of \textit{epistemic} uncertainty and allow us to quantitatively reason about uncertainty in model predictions. However, computing accurate posterior distributions becomes intractable as we use very complex models like deep neural nets, and current approaches require highly approximate inference methods that fall short of the promise of full Bayesian modeling in practice.

Bayesian methods also have a deep connection with the minimum description length (MDL) principle, a formalization of Occam's razor that casts learning as performing efficient data compression and has been widely used as a motivation for model selection techniques. 
Codes corresponding to maximum-a-posteriori estimators and Bayes marginalization have been commonly used within the MDL framework. 
However, other coding schemes have been proposed in MDL centered around achieving different notions of minimax optimality.
Interpreting coding schemes as predictive distributions, such methods can directly inspire prediction strategies that give conservative predictions and do not suffer from excessive overconfidence due to their minimax formulation. 

One such predictive distribution is the \textit{conditional normalized maximum likelihood} (CNML) \citep{Grunwald2007TheLearning, Rissanen2007ConditionalModels, Roos2008BayesianModels} model, also known as sequential NML or predictive NML \citep{Fogel2018UniversalLog-Loss}. 
To make a prediction on a new input, CNML considers every possible label and finds the model that best explains that label for the query point together with the training set. It then uses that corresponding model to assign probabilities for each input and normalizes to obtain a valid probability distribution. 
We will argue that the CNML prediction strategy can be useful for providing reliable uncertainty estimates on out-of-distribution inputs.
Intuitively, instead of relying on a learned model to extrapolate from the training set to the new (potentially out-of-distribution) input, CNML can obtain more reasonable predictive distributions by explicitly updating a model for each potential label of the particular test input and then asking ``given the training data, which labels would make sense for this input?'' 

While CNML provides compelling minimax regret guarantees,
practical instantiations have been exceptionally difficult, because computing predictions for a test point requires retraining the model on the test point \emph{concatenated with the entire training set}. With large models like deep neural networks, this can require hours of training for every prediction, rendering naive CNML schemes infeasible for practical use.

In this paper, we argue that prediction strategies inspired by CNML, which output conservative predictions that depend on models explicitly trained on the test input, can provide reasonable uncertainty estimates even when faced with out-of-distribution data.
To instantiate such a strategy tractably, we propose \textit{amortized CNML} (ACNML), a practical algorithm for approximating CNML utilizing approximate Bayesian inference. ACNML avoids the need to optimize over large datasets during inference by using an approximate posterior in place of the training set. 
We show that our proposed approach is compares favorably to number of prior techniques for uncertainty estimation on out-of-distribution inputs, and is substantially more feasible and computationally efficient than prior techniques for using CNML predictions with deep neural networks.



\section{Conditional Normalized Maximum Likelihood}

ACNML is motivated from the minimum description length (MDL) principle, which states that any regularities in a dataset can be exploited to compress it, and so learning is reformulated as encoding the data as efficiently as possible.
\citep{Rissanen1989StochasticTheory,  Grunwald2007TheLearning}. 
While MDL is typically described in terms of code lengths, we can associate codes with probability distributions, with the code length of an object corresponding to the negative log-likelihood under that probability distribution.
MDL was originally formulated in a generative setting where the goal is to code arbitrary data, we focus here on a supervised learning setting, where we assume the inputs are already known and our goal is to only encode/predict the labels.

\textbf{Normalized Maximum Likelihood.} 
Suppose we have a model class $\Theta$, where each $\theta \in \Theta$ corresponds to a conditional distribution $p_{\theta}(y\vert x)$. Let $\hat \theta(y_{1:n} \vert x_{1:n})$ denote the maximum likelihood estimator for a sequence of labels $y_{1:n}$ corresponding to inputs $x_{1:n}$ over all $\theta \in \Theta$.
Given a sequence of inputs $x_{1:n}$ and labels $y_{1:n}$, we can define a regret for a distribution over labels $q$ as 
\begin{align}
    R(q, y_{1:n}, x_{1:n}, \Theta) \! \defeq \! \log p_{\hat{\theta}(y_{1:n}\vert x_{1:n})}(y_{1:n}\vert x_{1:n}) \nonumber \\- \log q(y_{1:n}).
\end{align}
In relation to the MDL principle, this regret corresponds to the excess number of bits $q$ uses to encode the labels $y_{1:n}$ compared to the best distribution in the model class $\Theta$.
For any fixed input sequence, we can then define the \textit{normalized maximum likelihood distribution} (NML) as 
\begin{equation}
    p^{\text{NML}}(y_{1:n}\vert x_{1:n}) = \frac{p_{\hat \theta(y_{1:n} \vert x_{1:n})}(y_{1:n}\vert x_{1:n})}{\sum_{\tilde y_{1:n} \in \mathcal{Y}^n} p_{\hat \theta(\tilde{y}_{1:n}\vert x_{1:n})} (\tilde y_{1:n}\vert x_{1:n})}.
\end{equation}
The NML distribution can be shown to achieve minimax regret \citep{Shtarkov1987UniversalMessages, Rissanen1996FisherComplexity} as it achieves the same regret for all label sequences.
\begin{equation} 
    p^{\text{NML}} =\argmin_{q} \max_{y_{1:n} \in \mathcal Y^n} R(q, y_{1:n}, x_{1:n}, \Theta).
\end{equation}
This corresponds, in a sense, to an optimal coding scheme for sequences of labels of known fixed length $n$.

\textbf{Conditional NML.} Instead of making predictions across entire sequences of labels at once, NML can be adapted to the setting where we make predictions about only the next label
based on the previously seen data, resulting in \textit{conditional NML} (CNML) \citep{Rissanen2007ConditionalModels, Grunwald2007TheLearning, Fogel2018UniversalData}.
While several variations on CNML exist, we consider the following:
\begin{equation}\label{supervised-cnml}
p^{\text{CNML}}(y_n \vert x_n; x_{1:n-1}, y_{1:n-1}) \propto p_{\hat \theta(y_{1:n} \vert x_{1:n})}(y_n\vert x_n),
\end{equation}
which solves the minimax problem
\begin{equation} \label{eq:cnml-minimax-supervised}
    p^{\text{CNML}}= \argmin_q \max_{y_n} \log p_{\hat\theta(y_{1:n}\vert x_{1:n})}(y_n\vert x_n) - \log q(y_n).
\end{equation}
We note that the inner maximization is only over the next label $y_n$ that we are predicting, rather than the full sequence as before.
This prediction strategy is now amenable to our typical supervised learning setting, where $(x_{1:n-1}, y_{1:n-1})$ is our training set, and we want to output a predictive distribution over labels $y_n$ for a new test input $x_n$.

\textbf{CNML provides conservative predictions.} 
Here we motivate why CNML can provide reasonable uncertainty estimates for out-of-distribution inputs. For each query point, CNML considers each potential label and finds the model that would be most consistent with that label and with the training set. 
If that model assigns high probability to the label, then minimizing the worst-case regret forces CNML to assign relatively high probability to it.
Compared to simply letting a model trained only on the training set extrapolate to OOD inputs, we expect CNML to give more conservative predictions on OOD inputs, since it explicitly considers what would have happened if the new data point had been labeled with each possible label. 

\begin{wrapfigure}{r}{0.45\columnwidth}
\vspace{-10pt}
\includegraphics[width=1.0\linewidth]{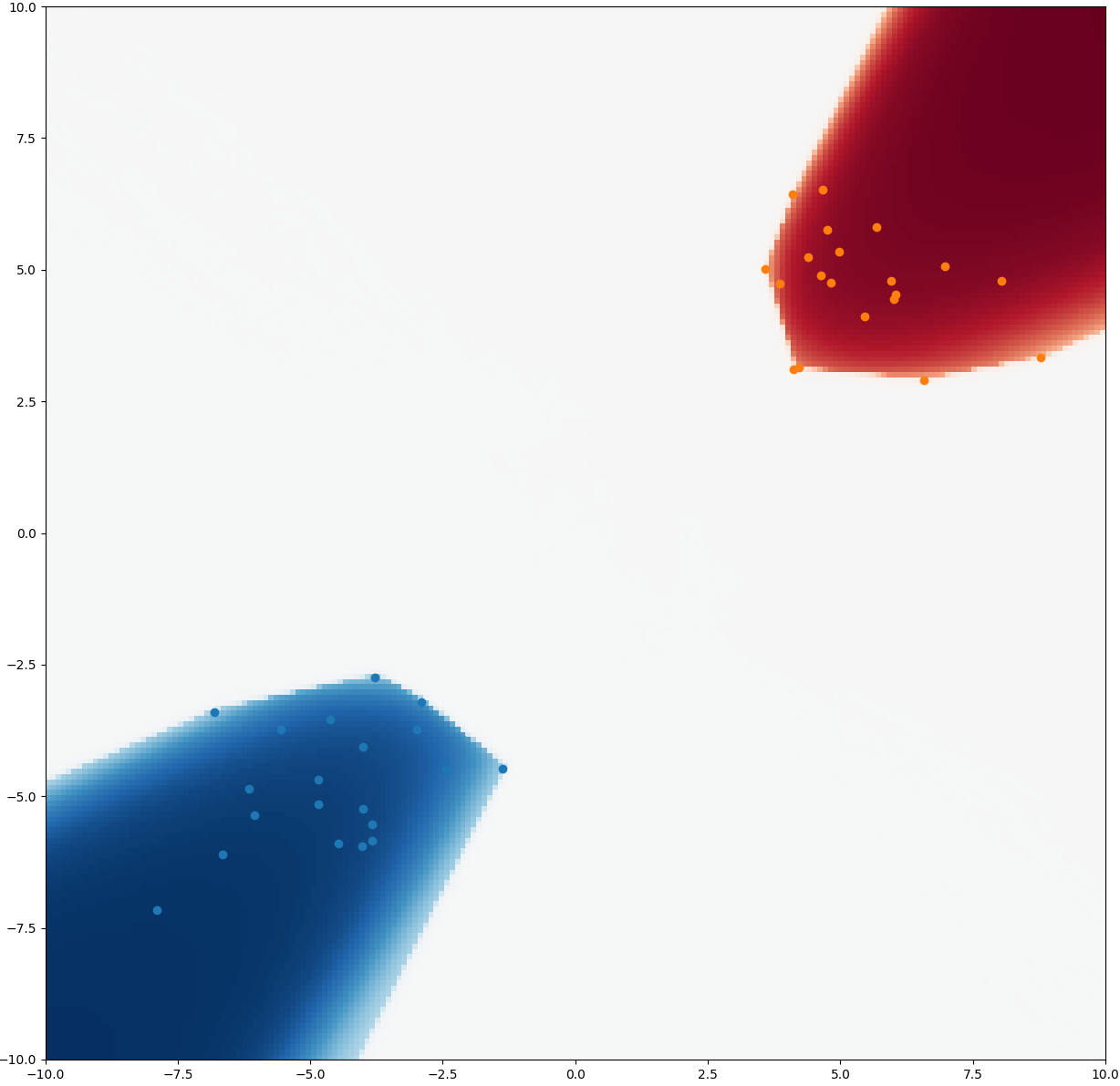}
\vspace{-20pt}

\caption{\label{fig:cnml-heatmap} CNML probabilities with a logistic regression model. CNML expresses high uncertainty and provides uniform predictions (indicated by the white color) on most of the input space away from the training set (shown in blue and orange dots).}
\vspace{-15pt}
\end{wrapfigure}

We use a 2D logistic regression example to illustrate CNML's conservative predictions, showing a heatmap of CNML probabilities in Figure~\ref{fig:cnml-heatmap}. CNML provides uniform predictions on most of the input space away from the training samples. In Figure~\ref{fig:cnml-basic}, we illustrate how CNML arrives at these predictions, showing the predictions for the parameters $\hat{\theta}_0$ and $\hat{\theta}_1$, corresponding to labeling the test point (shown in pink in Figure~\ref{fig:cnml-basic}, left) with either label 0 or 1.

However, CNML may be too conservative when the model class $\Theta$ is very expressive.
Na\"{i}vely applying CNML with large model classes can result in the per-label models fitting their labels for the query point arbitrarily well, such that CNML gives unhelpful uniform predictions even on inputs we would hope to reasonably extrapolate on. We see this in the 2D logistic regression example in Figure \ref{fig:cnml-heatmap}.
Thus, the model class $\Theta$ would need to be restricted in some form, for example by only considering parameters within a certain distance from the training set solution as a hard constraint.

\begin{figure}[t]
\centering
\includegraphics[width=\linewidth]{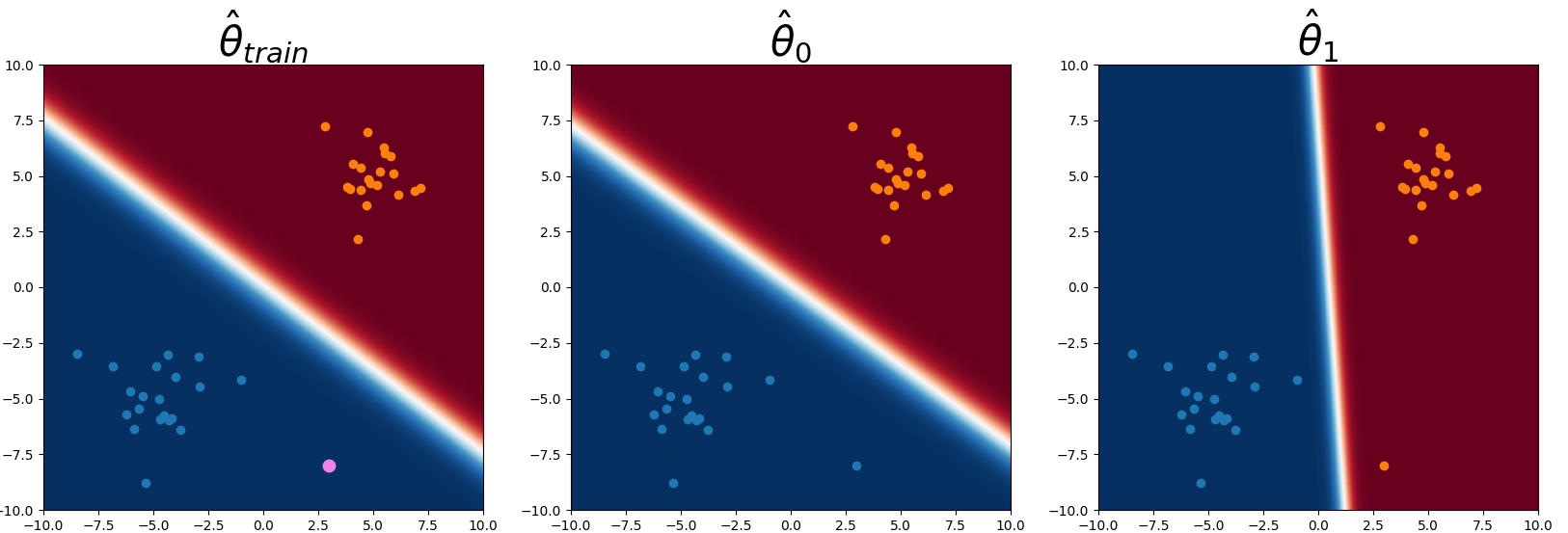}
\vspace{-20pt}
\caption{\label{fig:cnml-basic}
Given the labeled training set (blue and orange dots), we want to predict the label at the query input (shown in pink in the left image), which the training set MLE $\hat\theta_{\text{train}}$ confidently classifies as the blue class. However, CNML assigns a near-uniform prediction on the query point, as it computes new MLEs $\hat \theta_0$ and $\hat \theta_1$ (center and right images) by assigning different labels to the query point, and finds both labels are consistent with the training data. 
}
\vspace{-20pt}
\end{figure}

Another approach for controlling the expressivity of the model class is to generalize CNML to use \textit{regularized} estimators instead of maximum likelihood, resulting in normalized maximum a posteriori (NMAP) \citep{Kakade2006Worst-caseModels} codes.
Instead of using maximum likelihood parameters, NMAP selects $\hat \theta$s
to be the parameter that maximizes both data likelihood and a regularization term, or prior, over parameters, and we can define slightly altered notions of regret using these MAP estimators in all the previous equations to get a \textit{conditional normalized maximum a posteriori} distribution instead. 
See Appendix \ref{appendix:nmap} for completeness.

Going back to the logistic regression example, we plot heatmaps of CNMAP predictions in Figure \ref{fig:cnmap-heatmaps}, adding different amounts of L2 regularization to the logistic regression weights. As we add more regularization, the model class becomes effectively less expressive, and the CNMAP predictions become less conservative.

\begin{figure}[h]
    \centering
    \begin{subfigure}[t]{0.3\linewidth}
            \centering
            \includegraphics[width=\linewidth]{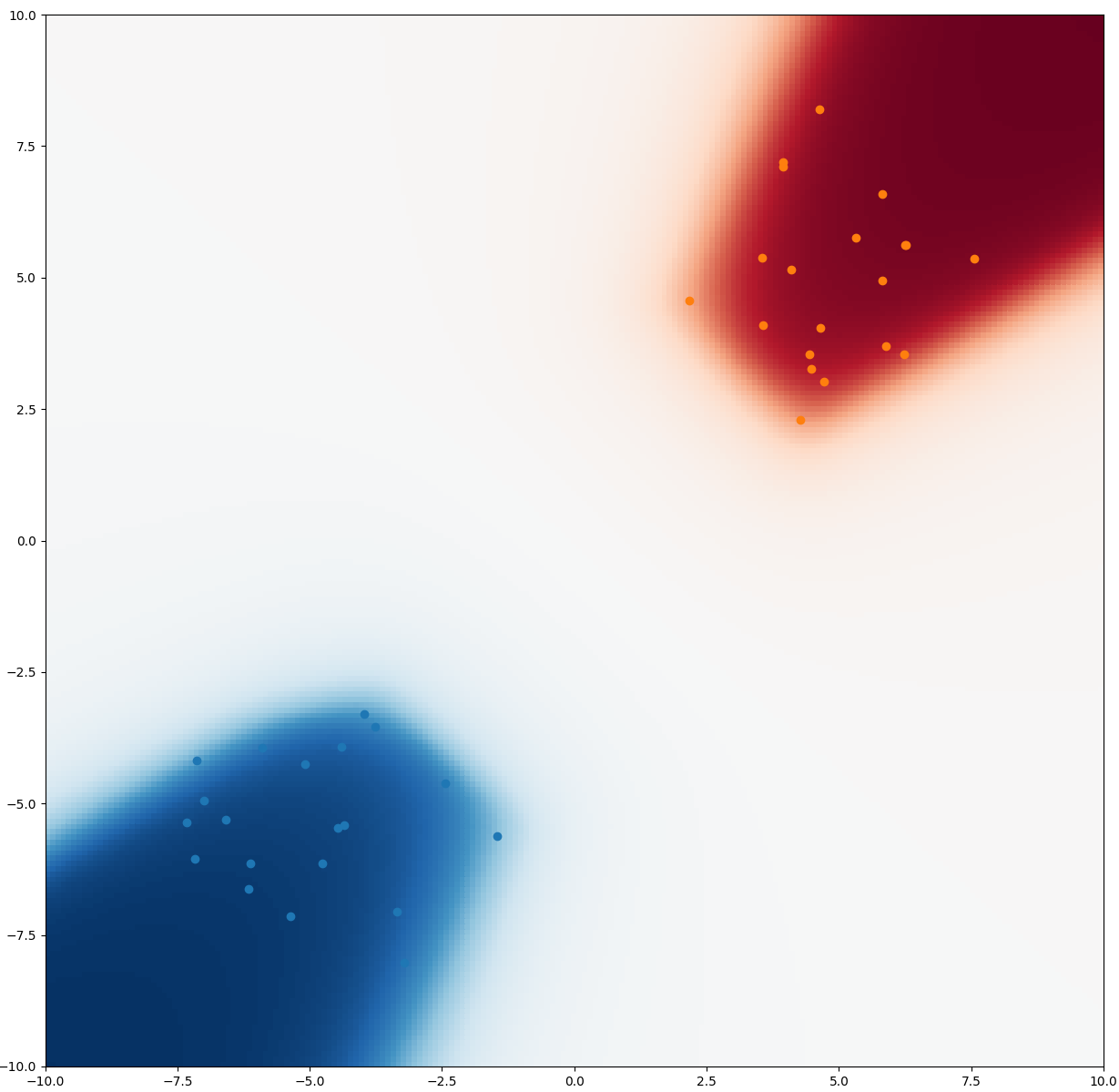}
            \caption{$\lambda=0.1$}
            \label{fig:cnml-illustration:mle}
        \end{subfigure}%
    \begin{subfigure}[t]{0.3\linewidth}
            \centering
            \includegraphics[width=\linewidth]{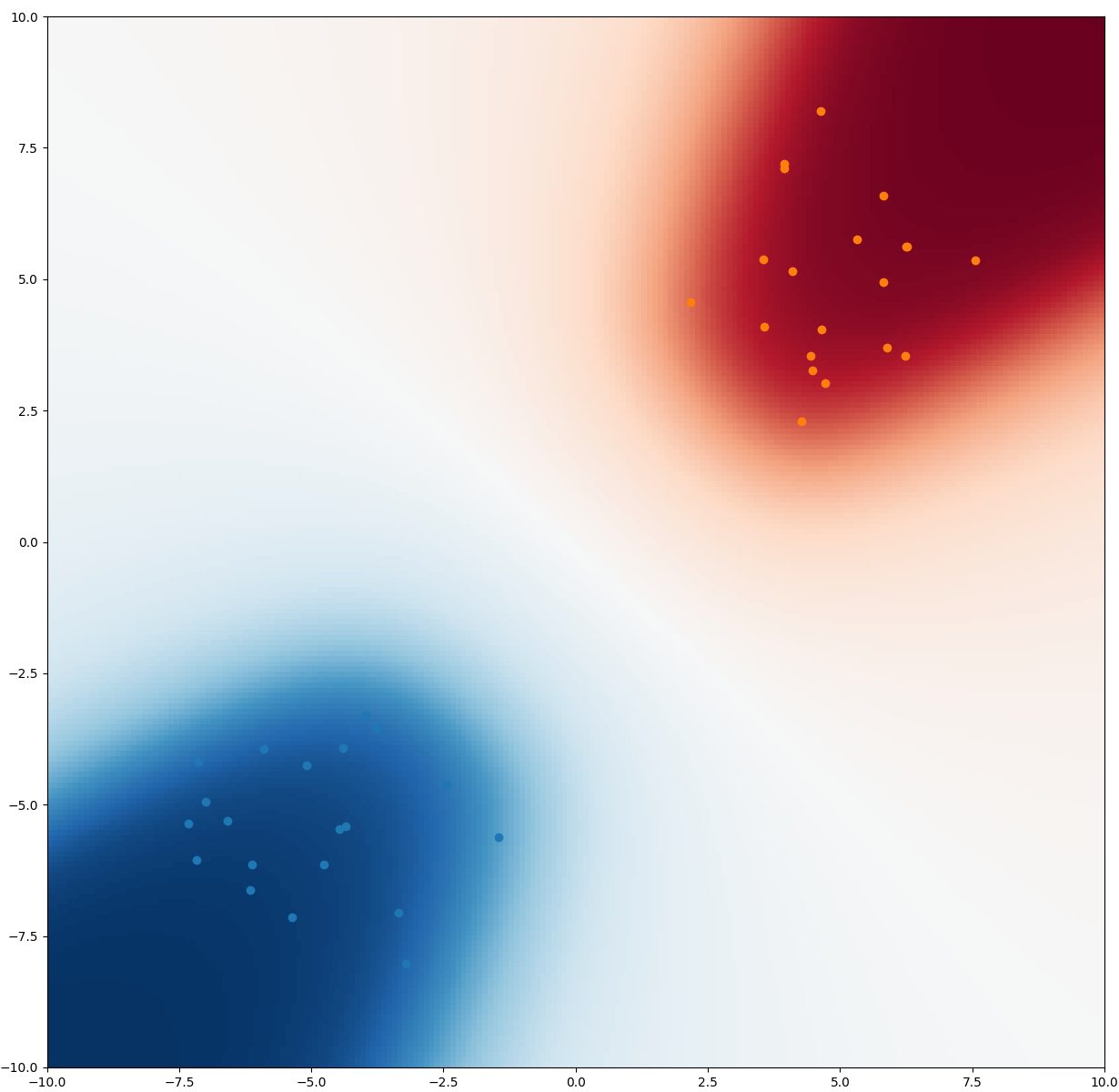}
            \caption{$\lambda=1$}
            \label{fig:cnml-illustration:cnml}
        \end{subfigure}%
    \begin{subfigure}[t]{0.3\linewidth}
            \centering
            \includegraphics[width=\linewidth]{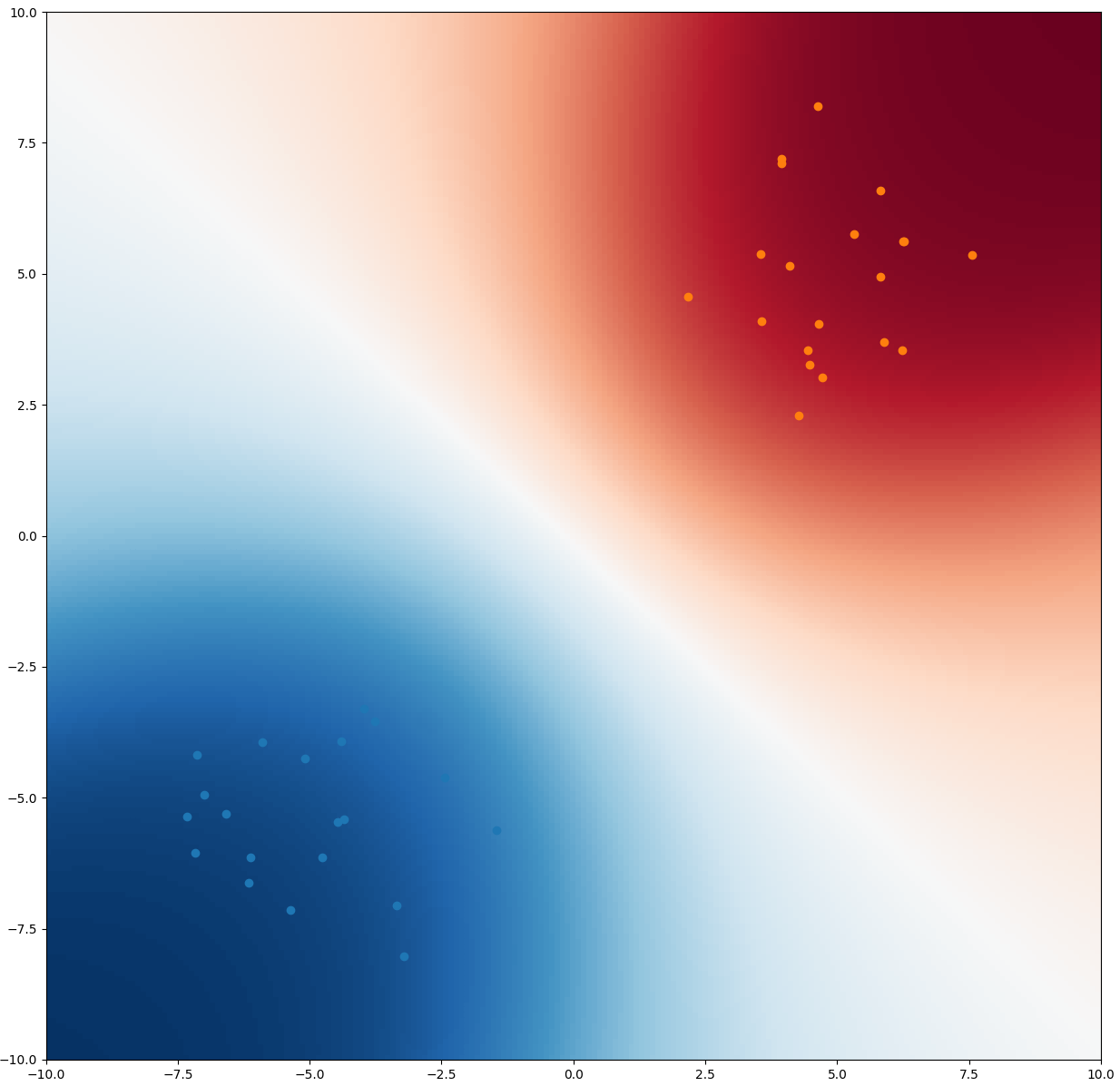}
            \caption{$\lambda=10$}
            \label{fig:cnml-illustration:nmap}
        \end{subfigure}%
    \caption{CNMAP probabilities with different levels of L2 regularization $\lambda \norm{w}_2^2$. Predictions are less conservative as $\lambda$ increases.}
    \label{fig:cnmap-heatmaps}
    \vspace{-10pt}
\end{figure}

\textbf{Computational costs of CNML.}
While we have argued that CNML can provide an appealing approach for uncertainty estimation for out-of-distribution inputs, it can be exceptionally impractical to instantiate, particularly with large models like neural networks, due to the prohibitive computational costs of computing the maximum likelihood estimators for each new input and label.
To evaluate the distribution on a new test point, one must solve a nonconvex optimization problem for each possible label, with each problem involving the entire training dataset along with the new test point.
This direct evaluation of CNML therefore becomes computationally infeasible with large datasets and high-capacity models, and further requires that the model carry around the entire training set even when it is deployed.
In settings where critical decisions must be made in real time, even running a single epoch of additional training would be infeasible. For this reason, NML-based methods have not gained much traction as a practical tool for improving the predictive performance of high-capacity models.

\section{Amortized CNML}
In this section, we derive our method, amortized conditional normalized maximum likelihood (ACNML), which provides a tractable approximation for CNML and CNMAP via approximate Bayesian inference.
Instead of directly computing maximum likelihood parameters over the query point and training set, our method uses an approximate posterior distribution over parameters to capture the necessary information about the training set, reducing the maximization to only the single new point. The computational cost at test-time therefore does not increase with training set size.



\subsection{Algorithm Derivation}
\textbf{Incorporating an exact posterior into CNML.}
Given a prior distribution $p(\theta)$, the Bayesian posterior likelihood conditioned on the training data is given by 
\begin{equation}
    p(\theta \vert x_{1:n-1}, y_{1:n-1}) \propto p(\theta)p_\theta(y_{1:n-1}\vert x_{1:n-1}).
\end{equation}
We can write the MAP estimators in the CNMAP distribution for a fixed query input $x_n$ as
\begin{align}
    \hat \theta_y = \argmax_{\theta \in \Theta} & \underbrace{\log p_{\theta}(y_{1:n-1}\vert x_{1:n-1}) + \log p(\theta)}_{\log p(\theta \vert x_{1:n-1}, y_{1:n-1})} \nonumber \\ & +\log p_{\theta}(y\vert x_n)\label{eq:cnml-mle}
\end{align}
We can thus replace the training data log-likelihood $p_{\theta}(y_{1:n-1}\vert x_{1:n-1})$ with the Bayesian posterior density $\log p(\theta\vert x_{1:n-1}, y_{1:n-1})$ when computing $\hat\theta_y$. 
We can also recover CNML as a special case of CNMAP by using a uniform prior, but as discussed previously, CNML with highly expressive model classes can lead to overly conservative predictions, so we will opt to use non-uniform priors that help control model complexity instead.
For example, we may use a zero-mean Gaussian prior $p(\theta)$ over our weights, corresponding to L2 regularization. 

\begin{algorithm}[t]
\caption{Amortized CNML (ACNML)}
\label{alg:acnml}
\begin{algorithmic}
\STATE \textbf{Input}: Model class $\Theta$, Training Data $(x_{1:n-1}, y_{1:n-1})$, Test Point: $x_n$, Classes (1, \ldots, k) 
\STATE \textbf{Output}: Predictive distribution $p(y \vert x_n)$ 
\STATE \textbf{Training}: Run approximate inference algorithm on training data $(x_{1:n-1}, y_{1:n-1})$ to get posterior density $q(\theta)$ 
\FORALL{possible labels $i \in (1, \ldots, k)$}
\STATE    Compute $\hat \theta_i = \argmax_{\theta} \log p_{\theta}(i \vert x_n) + \log q(\theta)$ \\
\ENDFOR
\STATE Return $p(y\vert x_n) = \frac{p_{\hat \theta_y}(y\vert x_n)}{\sum_{i=1}^k p_{\hat \theta_i}(i\vert x_n)}$ 
\end{algorithmic}
\end{algorithm}
\textbf{ACNML with an approximate posterior.}
Of course, the exact Bayesian likelihood is no easier to compute than the original training log likelihood. However, we can derive a tractable approximation by replacing the exact posterior $p(\theta\vert x_{1:n-1}, y_{1:n-1})$ with an approximate posterior $q(\theta)$ instead.
We can obtain an approximate posteriors via standard approximate Bayesian techniques such as variational inference or Laplace approximations.
We focus on Gaussian posterior approximations for computational efficiency, and discuss in Section~\ref{sec:theory} why this class of distributions provides a reasonable approximation for large datasets.

For practical purposes, we expect the approximate posterior log-likelihood to ensure the optimal $\hat \theta_y$ selected for each  label retains good performance on the training set.
By replacing the likelihood over the training data with the probability under an approximate posterior,
it becomes unnecessary to retain the training data at test time, only the parameters of the approximate distribution.
Optimization also becomes much simpler, as it no longer requires stochastic gradients, and the Gaussian posterior log density $\log q(\theta)$ serves as a strongly convex regularizer.

\textbf{ACNML algorithm summary}
A summary of the ACNML algorithm is presented in Algorithm~\ref{alg:acnml}. The training process for obtaining $q(\theta)$ only needs to be performed once on the training set, whereas the inference step is performed for each test point. However, this inference step only requires optimizing the model on a single data point with a regularizer provided by $\log q(\theta)$.

\subsection{Analysis of ACNML with Gaussian Posteriors} 
\label{sec:theory}
In this section, we argue that using a Gaussian approximate posterior in ACNML, which correspond to second-order approximations to the training set log-likelihood, suffices for accurately computing the CNML distributions when the training set is large.
The intuition is that for large training sets, the combined likelihoods of all the training points dominate over the single new test point, so the perturbed MLEs $\hat \theta_y$ remains close to the original training set MLE $\hat \theta$, letting us rely on local approximations to the training loss.  

Under simplifying assumptions of convexity and smoothness of the training losses, we can formalize this using the concept of \textit{influence functions}, which measure how the MLE (and more general $M$-estimators) for a dataset changes as the dataset were perturbed by reweighting inputs an infinitesimal amount. 
Recall that the maximum likelihood estimator for a dataset with $n$ datapoints $(x_{1:n}, y_{1:n})$ is given by 
\vspace{-4pt}
\begin{align}
    \hat \theta = \argmax_{\theta} \frac{1}{n} \sum_{i=1}^n \log p_{\theta}(y_i\vert x_i).
\end{align}
Influence functions analyze how $\hat \theta$ relates to the MLE of a perturbed dataset
\begin{align}
    \hat \theta_{x, y, \epsilon} = \argmax_{\theta} \left(\epsilon\log p_{\theta}(y\vert x) + \frac{1}{n}\sum_{i=1}^n \log p_{\theta}(y_i\vert x_i)\right),
\end{align}
where $\hat \theta_{x, y,\epsilon}$ is the new MLE if we perturb the training set by adding a datapoint $(x,y)$ with a weight $\epsilon$.
A classical result \citep{cook1982residuals}  shows that $\hat \theta_{x, y, \epsilon}$ is differentiable (under appropriate regularity conditions) with respect to $\epsilon$ with derivative given by the influence function
\begin{align}
    \frac{d\hat \theta_{x, y, \epsilon}}{d\epsilon}\lvert_{\epsilon=0} = -H_{\hat \theta}^{-1} \grad_\theta \log p_{\hat \theta}(y\vert x),
\end{align}
where $\hat \theta$ is the MLE for the original dataset and $H_{\hat \theta}$ the Hessian of the mean training set log-likelihood
evaluated at $\hat \theta$.
CNML computes the MLE after adding datapoint $(x,y)$ with equal weight as points in the training set, which is precisely $\hat \theta_{x, y, \epsilon}$ evaluated at $\epsilon = 1/n$.
Thus, for sufficiently large $n$,
a first order Taylor expansion around $\hat\theta$ should be accurate and the new parameter can be estimated by
\begin{align} 
    \tilde\theta_{x,y} = \hat \theta  - \frac{1}{n} H_{\hat \theta}^{-1} \nabla_\theta \log p_{\hat \theta} (y\vert x), \label{eq:influence-mle}
\end{align}
which is equivalent to solving
\begin{align}
    \tilde \theta_{x,y} = \argmax_\theta & \frac{1}{n}(\theta - \hat \theta)^T \nabla_\theta \log p_{\hat \theta} (y\vert x) \nonumber \\ &+ \frac{1}{2}(\theta - \hat \theta)^T H_{\hat \theta}(\theta - \hat \theta).
\end{align}
This suggests that, with large training datasets, the perturbed MLE parameters $\hat \theta_y$ in Equation $\ref{eq:cnml-mle}$ can be approximated accurately using a quadratic approximation to the training log-likelihood, corresponding to a Gaussian posterior obtained via a Laplace approximation.
We can explicitly quantify the accuracy of this approximation in the theorem below, which is based on Theorem 1 from \citet{giordano2019swiss}, with full details and proof in Appendix \ref{appendix:ij-expanded-details}.
\begin{theorem} \label{adapted-theorem} (Adapted from \citet{giordano2019swiss})
    Consider a training set with $n$ datapoints and an additional datapoint $ (x,y)$.
    Assume assumptions 1-5 hold with constants $C_{op}, C_{\text{IJ}}, \Delta_\delta$ as defined in Appendix \ref{appendix:ij-expanded-details}.
    Let $\hat \theta_{x,y}$ denote the exact MLE if we had appended $(x,y)$ to the training set, and $\tilde \theta_{x,y}$ the parameter obtained via the approximation in Equation $\ref{eq:influence-mle}$.
        Let
    \begin{align}
       \delta = \frac{\sup_{\theta \in \Theta} \max \left\{ \norm{\grad_\theta \log p_{\theta}(y\vert x)}_1, \norm{\grad_{\theta}^2 \log p_{\theta}(y\vert x)}_1\right\} }{n+2}.
    \end{align}
    If $\delta \leq \Delta_\delta$, then
    \begin{align}
        \lVert \hat \theta_{x,y} - \tilde \theta_{x,y}\rVert_2 \leq 2C_{op}^2 C_{\text{IJ}} \delta^2.
    \end{align}
\end{theorem}
Given such a bound on how accurately we estimate new parameters, we can explicitly quantify the accuracy of the CNML approximation, with proof in Appendix \ref{appendix:ij-expanded-details}.
\begin{prop} \label{cnml-logit-accuracy-bound}
Let $\hat \theta_{x,y}$ and $\tilde \theta_{x,y}$ be the exact and approximate MLEs respectively, after appending the datapoint $(x,y)$ to the training set, and assume $\lVert \hat \theta_{x,y} - \tilde \theta_{x,y}\rVert \leq \delta$ for all $y$. 
Further suppose $\log p_{\theta}(y\vert x)$ is $L$-Lipschitz in $\theta$.

Let $p_{\text{CNML}}(y) \propto p_{\hat \theta_{x,y}}(y\vert x)$ and $p_{\text{ACNML}}(y) \propto p_{\tilde \theta_{x,y}}(y\vert x)$ denote the exact CNML and approximate CNML distributions respectively.
We then have
\begin{align}
    \sup_{y} \abs{\log p_{\text{CNML}}(y) - \log p_{\text{ACNML}}(y)} \leq 2 L\delta.
\end{align}
\end{prop}
Theorem 3.1 and Proposition 3.2 suggest the approximation given by ACNML will be increasingly close to the exact CNML distribution as the training set size $n$ grows. 
However, this formal theoretical result only holds for sufficiently large datasets and under assumptions including smoothness and convexity of the training loss, so does not necessarily hold in practical settings with deep neural networks.

To interpret how different training points influence the predictions of neural networks, \citet{pmlr-v70-koh17a} showed that influence function approximations were able to provide useful predictions for estimating leave-one-out retraining with deep convolutional neural networks. This closely resembles the conditions we encounter when computing parameters for each label of the query point with ACNML, with the key difference being that ACNML \textit{adds} a datapoint while leave-one-out retraining \textit{removes} one.
Their empirical results suggest these second-order approximations to the training loss, corresponding to Gaussian approximations in ACNML, may suffice to yield useful predictions about how parameters change when the query point is added, despite lacking formal guarantees with deep neural networks.

\section{Related Work}
Minimum description length has been used to motivate neural network methods dating back to \citet{Hinton1993KeepingWeights}, who treat description length as a regularizer to mitigate overfitting. The idea of preferring flat minima \citep{Hochreiter1997FlatMinima} also has its origins in the MDL framework, as it allows a coarser discretization of the weights (and thus fewer bits needed). 
Bayesian methods average the predictions of different models sampled from the posterior distribution and typically serve as the starting point for uncertainty estimation in deep networks. A common approach is to use simple tractable distributions to approximate the true posterior \citep{Hoffman2013StochasticInference, Blundell2015WeightNetworks, ritter2018scalable}. 
Recent work \citep{Maddox2019ALearning, Dusenberry2020EfficientFactors} has shown simple Gaussian posterior approximations are able to achieve well-calibrated predictions with marginalization.
ACNML utilizes these approximate posterior methods, but in contrast to the Bayesian methods, where the posterior is used to efficiently sample models for Bayesian model averaging, ACNML uses the posterior density to enable efficient optimization without needing to retain the training data. 

\citet{ovadia2019trust} evaluate various proposed methods for uncertainty estimates in deep learning under different types of distribution shift, finding that good calibration on in-distribution points did not necessarily indicate good calibration under distribution shift, and that methods relying on marginalizing predictions over multiple models \citep{Lakshminarayanan2016SimpleEnsembles, Srivastava2014Dropout:Overfitting} gave better uncertainty estimates under distribution shift than other techniques. 
In our experiments, we show that our method ACNML maintains much better calibration under distribution shift than prior methods.

Similarly to ACNML, Test Time Training (TTT) \citep{sun19ttt} updates the model on test inputs to improve out-of-distribution performance. One key differences is that TTT relies on an auxiliary self-supervised task to solve on the new test point, and so requires domain knowledge to specify a nontrivial task that is useful for predictions. Additionally, the goal of TTT was to enable \textit{more accurate} prediction under distribution shift, whereas our goal with ACNML was to provide more reliable \textit{uncertainty estimates}.




Perhaps most closely related to our work, \citet{Fogel2018UniversalLog-Loss} advocate for the use of the CNML distribution in the context of supervised learning (under the name predictive NML), citing its minimax properties.
\citet{Bibas2019DeepNetworks} estimate the CNML distribution with deep networks by finetuning the last layers of the network on every test input and label combination appended to the training set.
Since this finetuning procedure trains for several epochs, it is very computationally intensive at test-time and requires continued access to the entire training set when evaluating.
In contrast, our method amortizes this procedure by condensing the information in the training data into a distribution over parameters, allowing for much faster test-time inference without needing the training data.

In the analysis for our approximation, we draw connections to influence functions \citep{cook1982residuals}, which have been studied as asymptotic approximations to how $M$-estimators change when perturbing a dataset.
In deep learning, \citet{pmlr-v70-koh17a} advocated for using influence functions to interpret neural nets, generate adversarial examples, and diagnose errors in datasets.
We use a theorem from \citet{giordano2019swiss}, which broadened the necessary assumptions for these infinitisemal approximations to be accurate and provides explicit guarantees for specific datasets rather than simply asymptotic results.

\begin{figure*}[h!]
\centering
\begin{subfigure}[t]{0.48\linewidth}
    \includegraphics[width=\linewidth]{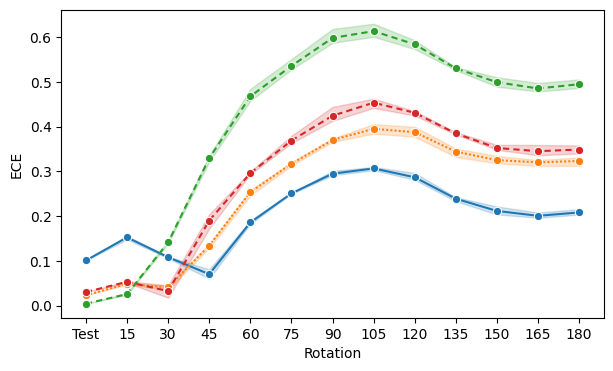}
    \vspace{-16pt}
    \caption{Rotated MNIST ECEs (lower is better)}
    \label{fig:mnist-ece-line}
\end{subfigure}%
\begin{subfigure}[t]{.48\linewidth}
    \includegraphics[width=\linewidth]{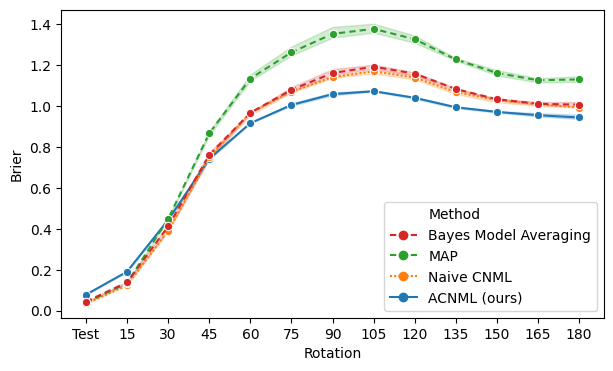}
    \vspace{-16pt}
    \caption{Rotated MNIST Brier Scores (lower is better)}
    \label{fig:mnist-brier-line}
\end{subfigure}%
\vspace{-6pt}
\caption{ACNML compared against its Bayesian counterpart,
the deterministic MAP baseline, and naive CNML on rotated MNIST. We plot means and standard deviations across 3 seeds. We see that ACNML (blue, solid lines)
achieves lower ECE as the distribution shift becomes more severe and accuracy decreases, as well as better Brier scores than other methods.}
\label{fig:rotated-mnist-lineplots}
\end{figure*}

\section{Experiments}
Our experiments aim to evaluate how trustworthy the uncertainty estimates provided by ACNML are under different levels of distribution shift. 
Following \citet{ovadia2019trust}, we compare uncertainty estimation across different methods using Brier score and expected calibration error (ECE) \citep{Naeini2015ObtainingBinning}. 
Brier score is a proper scoring rule, which captures both how accurate and how calibrated the predictions are, while
ECE assesses calibration by directly measuring how closely the predicted confidence corresponds to empirical accuracy.
We show that our method is able to to significantly outperform prior works in terms of calibration when distribution shifts became more extreme.
While severe distribution shifts mean all methods test perform poorly in terms of accuracy, ACNML is at least able to more reliably indicate when the predictions may be incorrect.

In principle, any method for computing a tractable posterior over parameters can be used with ACNML, and we demonstrate this flexibility by implementing ACNML on top of several different approximate posteriors. By using the exact same posteriors, we can directly compare how uncertainty estimates given by ACNML relate to those of the corresponding Bayesian method.

For each model, we report results across 3 seeds. 
as well as showing reliability diagrams \citep{Guo2017OnNetworks} to further qualitatively assess calibration.
For reliability diagrams, we sort data points by confidence and divide them into twenty equal sized buckets, plotting the mean accuracy against the mean confidence for each bucket. This allows to see qualitatively see how well the confidence of the prediction relates to the actual accuracy, as well as showing how the confidences are distributed for each method.

\begin{figure*}[t]
    \centering
        \begin{subfigure}[t]{0.33\linewidth}
            \centering
            \includegraphics[width=\columnwidth]{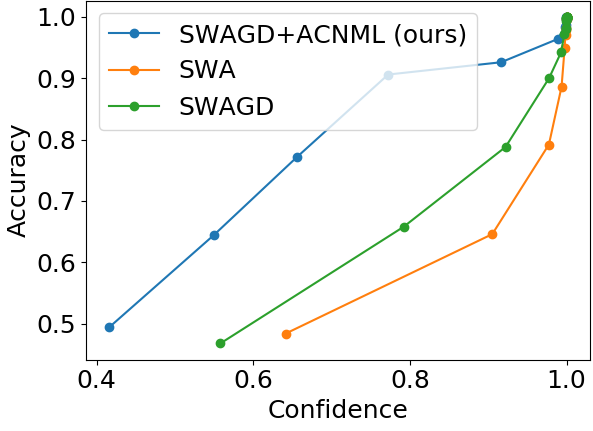}
            \caption{CIFAR10 Test}
            \label{fig:reliability:cifar-vgg}
        \end{subfigure}%
        \begin{subfigure}[t]{0.33\linewidth}
            \centering
            \includegraphics[width=\columnwidth]{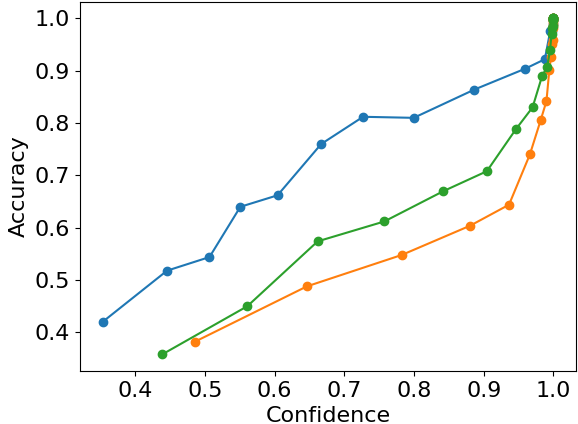}
            \caption{CIFAR10-C Corruption Level 3 }
            \label{fig:reliability:stl-vgg}
        \end{subfigure}%
        \begin{subfigure}[t]{0.33\linewidth}
            \centering
            \includegraphics[width=\columnwidth]{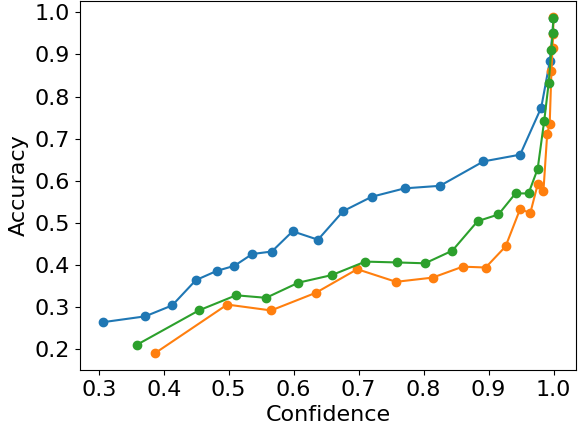}
            \caption{CIFAR10-C Corruption Level 5 }
            \label{fig:reliability:stl-wrn}
        \end{subfigure}%
        \vspace{-6pt}
        \caption{Reliability diagrams plotting confidence vs. accuracy for CIFAR10 in-distribution and OOD data. ACNML provides more conservative predictions than other methods, resulting in better calibration on OOD inputs. For OOD tasks, we show results for the Gaussian blur corruption at levels 3 and 5, with level 5 corresponding to a higher amount of corruption.
        Each point shows the mean confidence and accuracy within a bucket, so the spread of points along the $x$-axis shows that ACNML makes more low confidence predictions than other methods.}
    \vspace{-10pt}
    \label{fig:all_reliability_diagrams}
\end{figure*}

\begin{figure*}[t]
    \centering
\begin{subfigure}[t]{0.48\linewidth}
    \includegraphics[width=\linewidth]{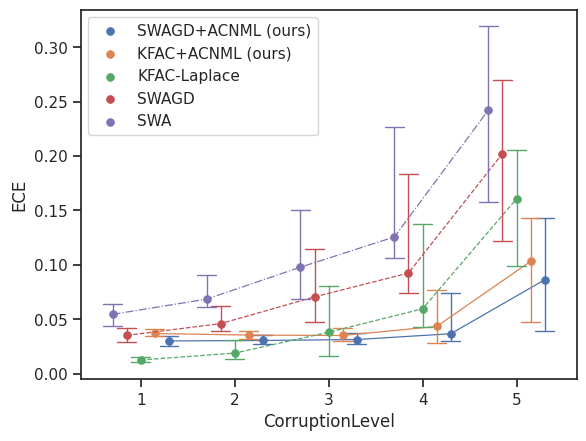}
    \caption{CIFAR10C VGG16 ECEs (lower is better)}
    \label{fig:cifarc-vgg-ece-line}
\end{subfigure}
\begin{subfigure}[t]{.47\linewidth}
    \includegraphics[width=\linewidth]{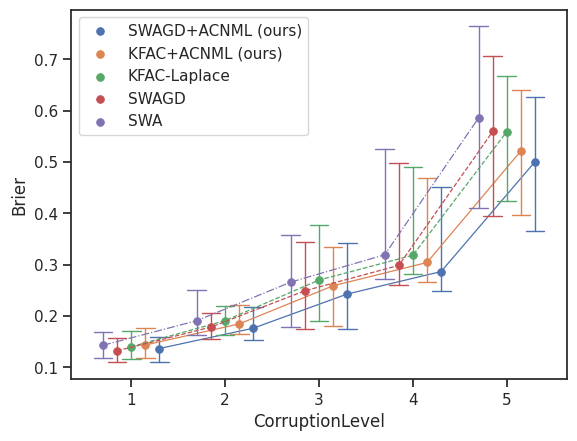}
    \caption{CIFAR10C VGG16 Brier Scores (lower is better)}
    \label{fig:cifarc-vgg-brier-line}
\end{subfigure}%
\vspace{-8pt}
\caption{ACNML compared against corresponding Bayesian methods and the deterministic MAP baseline on out-of-distribution CIFAR10-Corrupted datsets. We plot medians and 95\% confidence intervals across all corruptions. We see that ACNML methods (solid lines) achieve much lower ECE at higher corruption values, as well as better Brier scores than other methods.
}
\vspace{-10pt}
\label{fig:cifarc-lineplots-vgg}
\end{figure*}

\noindent \textbf{Rotated MNIST.} 
We first consider the rotated MNIST task, where out-of-distribution inputs are generated by rotating images from the MNIST test set, with higher levels rotation corresponding to more distribution shift.
Here, ACNML is implemented on top of Bayes-by-backprop~\citep{Blundell2015WeightNetworks}, and we compare to the MAP estimate and Bayes model averaging with the same posterior. 

We see in Figure \ref{fig:rotated-mnist-lineplots} that for higher levels of rotation, corresponding to more out-of-distribution inputs, that ACNML exhibits \textbf{substantial improvements in calibration} as measured by the ECE metric, as well as improved Brier scores.
However, on the in-distribution test set and the lowest levels of rotation where the models still predict accurately, ACNML's predictions are overly conservative, leading to underconfident predictions and worse calibration than other methods.
In general, this agrees with what we expect from ACNML: the predictions are more conservative across the board, which does not necessarily improve results in-distribution, particularly for easy domains like MNIST, but offer considerable improvements in calibration for out-of-distribution inputs where errors are prevalent.

We additionally compare to a much more computationally expensive instantiation of CNML used by \citet{Bibas2019ARegression} (denoted naive CNML in Figure \ref{fig:rotated-mnist-lineplots}), which directly finetunes for several epochs using the training set to obtain the optimal parameters for each query point and label, rather than using the approximate posterior like ACNML does. This direct instantiation of CNML improves over the MAP solution in terms of Brier score and calibration on the OOD inputs. However, it is computationally prohibitive, to the point where we were unable to evaluate it on the more complex datasets. 
On MNIST, each prediction with naive CNML was hundreds of times slower than with ACNML, as shown in Table \ref{tab:timing}. 
We also find ACNML is overall more conservative when using this particular posterior approximation, resulting in better calibration on more OOD inputs (see Appendix \ref{appendix:cnml-vs-acnml} for more detailed comparisons between ACNML and na\"{i}ve CNML). 

\noindent \textbf{CIFAR Corruptions.} 
We use CIFAR10 \citep{Krizhevsky2012LearningImages} for training and in-distribution testing, and evaluate uncertainty estimates under distribution shift using the CIFAR10-Corrupted \citep{hendrycks2019benchmarking} datasets, which apply different severities of 15 common corruptions to the test set images. 
We can thus assess calibration over a wide variety of distribution shifts, as well as how calibration degrades as distribution shift increases.

We show results here using the VGG16 \citep{Simonyan2014VeryRecognition} architecture.
To compute approximate posteriors, we use Stochastic Weight Averaging - Gaussian (SWAG) \citep{Maddox2019ALearning}, and KFAC-Laplace \citep{ritter2018scalable}.
SWAG computes a posterior by fitting a Gaussian distribution to the trajectory of SGD iterates. For simplicity and computational efficiency, we instantiate ACNML with the SWAG-D variant, which uses a Gaussian with diagonal covariance.
KFAC-Laplace uses a Gaussian posterior approximation with the MAP solution as the mean and the inverse Hessian of the loss as covariance, approximating the Hessian using KFAC \citep{martens2015optimizing}.

Focusing on the most direct comparisons, we compare against the MAP solution for the given posterior, which is equivalent to Stochastic Weight Averaging (SWA) \citep{Izmailov2018AveragingGeneralization},
and Bayes model averaging with SWAGD and KFAC-Laplace, which provide apples-to-apples comparisons to the two versions of our method that directly utilize the same posteriors from these prior approaches.

Examining the reliability diagrams in \autoref{fig:all_reliability_diagrams}, we can qualitatively see that ACNML provides more conservative (less confident) predictions than other methods across different levels of corruption. 
On out-of-distribution inputs, where accuracy degrades, we see that ACNML's conservative predictions lead to many better calibrated low-confidence predictions, while other methods drastically overestimate confidence. Thus, ACNML's confidence estimates are still able reliably indicate when predictions are likely to be incorrect even on OOD inputs. ACNML is however slightly \textit{underconfident} on the in-distribution CIFAR10 test set, while other methods err on the side of being overconfident.

In Figure \ref{fig:cifarc-lineplots-vgg}, we can quantitatively compare the calibration of different methods for different levels of corruption.
ACNML variants provide \textbf{much better calibration on the more severe corruptions} than other methods while also performing slightly better in terms of Brier score. All methods perform similarly in terms of accuracy in all domains, and we find that ACNML's more conservative estimates also perform competitively with Bayesian methods in Brier score, and ECE on the in-distribution test set as well (see Table \ref{tab:nll-acc-ece-indist} in Appendix \ref{appendix:more-exp-results}).
We include additional comparisons across other methods and architectures in Appendix \ref{appendix:more-exp-results}.
\begin{table}[h]

\scalebox{0.8}{
\begin{tabular}{|c|c|c|c|}
			\hline 
			& MNIST MLP  & VGG16 &  WRN28x10 \\ 
			\hline 
			ACNML (ours) & 0.08s & 0.37s & 1.1s \\ 
			na\"ive CNML (per epoch) & 13.83s & 102.0s & 359.1s \\ 
			feedforward inference & 0.0001s & 0.0013s & 0.004s \\ 
			\hline 
		\end{tabular} 
}
\caption{Inference time per input (in seconds). 
}
\vspace{-15pt}
\label{tab:timing}
\end{table}

\noindent \textbf{Timing Comparison vs. standard CNML}. In \autoref{tab:timing}, we examine the computational costs of our method.
We compare against a na\"ive implementation of CNML that fine-tunes for $N$ epochs on each test point and label, as in \citet{Bibas2019DeepNetworks}. 
In total, predicting a single input with $k$ possible labels involves running $kN$ epochs of training. While ACNML is over two orders of magnitude faster than na\"ive CNML even with just a single epoch of training (our experiments with naive CNML on MNIST used 5 epochs),
it is still slower than standard inference. 
The computational requirements of our method also scale linearly with the number of classes, but are constant with respect to dataset size.
Timing experiments are run using a single NVIDIA 1080Ti, using MNIST for the MNIST MLP timing results and using CIFAR10 for VGG16 and WideResNet28x10, with no parallelization over data points.

\section{Discussion}
In this paper, we present amortized CNML (ACNML)
as an alternative to Bayesian marginalization for obtaining reliable uncertainty estimates and calibrated predictions under distribution shift. 
The CNML distribution is a theoretically well-motivated strategy derived from the MDL principle with strong minimax optimality properties, but actually evaluating this distribution is computationally daunting. 
ACNML utilizes approximate Bayesian posteriors to tractably approximate it, can be instantiated on top of a wide range of approximate Bayesian methods, and provides much better calibrated predictions than other methods as the inputs become more out-of-distribution.
We view ACNML as a step towards practical uncertainty aware predictions that would be essential for real-world decision making.
Future work could further expand on our proposed method, for example by combining ACNML with more complex and expressive posterior approximations.
In particular, training losses are highly non-convex and have many local minima, so incorporating local approximations around \textit{multiple} diverse minima could allow for even more reliable uncertainty estimation.
More broadly, tractable algorithms inspired by ACNML could in the future provide for substantial improvement in our ability to produce accurate and reliable confidence estimates on out-of-distribution inputs, improving the reliability and safety of learning systems.
\
\bibliographystyle{abbrvnat}

\bibliography{bib}

\appendix
\appendix
\newpage
\hspace{0.1pt}
\newpage
\section{Experimental Details}\label{appendix:exp-details}

For obtaining approximate posteriors with SWAG and KFAC-Laplace, we follow the exact training procedures given in \citet{Maddox2019ALearning}.
We then implement ACNML on top of the diagonal SWAG posterior and the KFAC-Laplace posterior.

The variance of the SWAG posterior depends in a complex way on the learning rate and gradient covariances. To account for this, we introduce an additional temperature hyperparameter $\alpha$ and solve for the ACNML approximation using
\begin{equation}
    \theta^* = \argmax_{\theta \in \Theta} \log p_{\theta}(y_n \vert x_n) + \frac{1}{\alpha}\log q(\theta). 
\end{equation}
To calibrate $\alpha$, we can calculate the CNML distribution using a validation set, by training on the entire training set and the validation point, and then selecting $\alpha$ such that our ACNML procedure produces similar likelihoods.
We can also treat $\alpha$ as a tunable hyperparameter and select it using a validation set, similarly to how temperature scaling \citep{Guo2017OnNetworks} is used to achieve better calibration for prediction, or how the relative weighting of priors and likelihoods are used in generalized Bayesian inference \citep{Vovk1990AggregatingStrategies} or safe Bayesian inference \citep{Grunwald2017InconsistencyIt} as a way to deal with model misspecification.
For our experiments using the SWAGD posterior, we heuristically tune $\alpha$ to be as large as possible without degrading the accuracy compared to the MAP solution. 
Note, however, that this procedure is specific to the particular way in which SWAG estimates the parameter distribution, and any posterior inference procedure that explicitly approximates the posterior likelihood (e.g., \citet{Blundell2015WeightNetworks}) would not require this step.
To select $\alpha$ for each model class, we swept over values $[0.25, 0.5, 1, 1.5, 2]$ and selected the highest value such that accuracy and NLL on the validation set did not degrade significantly compared to SWA. 
For VGG16, we use $\alpha=0.5$ and for WideResNet28x10, we used $\alpha=1.5$.

\textbf{ACNML Optimization Details}: 
With our approximate posterior $q(\theta)$ being a Gaussian with covariance $\Sigma$, we approximately compute the MAP solution for each label $y$ as per Algorithm~\ref{alg:acnml} by initializing $\theta_0$ to be the posterior mean and iterating
\begin{equation}
    \theta_{t+1} = \theta_t + \epsilon_t \Sigma (\alpha \nabla \log p_{\theta_t}(y\vert x_n) + \nabla \log q(\theta_t)),
\end{equation}
using the covariance as a preconditioner. 
Similarly to the influence function calculation for the post update parameters discussed in section \ref{sec:theory}, this corresponds to taking approximate Newton steps at each iteration, using the Hessian approximation of the training set given by our approximate posterior.
For our experiments, we used a constant step size $\epsilon=0.5$ for the SWAG-D and BBP posteriors, and $\epsilon=0.25$ with KFAC-Laplace. 
We empirically found that 5 steps was often enough to find an approximate stationary point with the SWAG-D posterior, and 10 steps for the KFAC-Laplace posterior. 

For the reliability diagrams in Figure \ref{fig:all_reliability_diagrams}, we again follow the procedure used by \citet{Maddox2019ALearning}. 
We first divide the points into twenty bins uniformly based on confidence (each bin has the same number of points), then plot the mean accuracy vs mean confidence within each bin. This differs from the reliability diagrams used by \citet{Guo2017OnNetworks}, where they divide the range of confidence values into bins uniformly, resulting in unevenly filled bins.

For our expected calibration error (ECE) numbers, we use the same bins as computed for our reliability diagrams, and compute 
\begin{equation}
    ECE = \sum_{i=1}^K P(i) \cdot \abs{o_i - e_i},
\end{equation}
where $P(i)$ is the empirical probability a randomly chosen point lies in bin $i$, $o_i$ is the accuracy within bin $i$, and $e_i$ is the average confidence in bin $i$.

We adapted the SWAG authors' implementation at \hyperlink{https://github.com/wjmaddox/swa_gaussian}{https://github.com/wjmaddox/swa\_gaussian} to include the ACNML procedure for test time evaluation, and include a copy of the modified codebase in the supplementary materials with instructions on how to reproduce our experiments.
Experiments were conducted using a mix of local GPU servers and Google Cloud Program compute resources.

\textbf{MNIST Experimental Details}: For the MNIST experiments, we used a feedforward network with 2 hidden layers of size $1200$, with no data augmentation. 
The posterior is factored as independent Gaussians for each parameter, with the prior for each parameter being a zero-mean Gaussian with standard deviation 0.1.

We include expanded results with additional metrics in Figure \ref{fig:mnist-lineplots-expanded}.

\section{Further Experimental Results and Comparisons on CIFAR10} \label{appendix:more-exp-results}
In addition to the comparisons in the main paper, we additionally compare to SWA-Gaussian (SWAG), which uses a more expressive posterior than SWAG-D, SWA with Monte Carlo Dropout \citep{Gal2015DropoutLearning} (SWA-Drop), and SWA-Ensemble, which averages the predictions of independent runs of SWA as with regular deep ensembles \citep{Lakshminarayanan2016SimpleEnsembles}.
For reference, we show in-distribution performance of all methods in Table $\ref{tab:nll-acc-ece-indist}$.
Overall, performance differences between all methods are quite small, and ACNML's conservative predictions do not improve on NLL or ECE over some baselines on in-distribution performance, which is to be expected, since the main aim of our method is produce more calibrated predictions on \textbf{out-of-distribution} tasks.

For all ensembling or Bayesian marginalization methods, we marginalize over 30 model samples.

For completeness, we show expanded results on CIFAR10-Corrupted in Figures \ref{fig:cifarc-results-vgg}, \ref{fig:cifarc-results-vgg-drop}, and \ref{fig:cifarc-results-wrn}, which include additional baselines and metrics.
ACNML consistently achieves significantly better ECE than prior methods on the more severe corruptions, and generally comparable or slightly better NLL and Brier scores to the best performing baselines.
With the same architecture, all methods generally have very similar accuracy, with the exception of SWA-Ensemble slightly outperforming better than other methods in accuracy.

While evaluating MC-Dropout, we found that adding dropout before each layer in VGG16 (labelled VGG16Drop in \ref{fig:cifarc-results-vgg-drop}) significantly improved performance on CIFAR10-C.
For fair comparisons, we reran all methods with the VGG16Drop architecture as well. We again find that ACNML performs the best in terms of calibration on the more severe corruptions.

\begin{table*}
\centering
\scalebox{.7}{
\begin{tabular}{|c|ccc|ccc|}
			\hline 
			\multicolumn{1}{|c}{CIFAR10 Results} & \multicolumn{3}{|c}{VGG16} & \multicolumn{3}{|c|}{WideResNet28x10}\\
			\hline
		     & NLL & Accuracy  & ECE & NLL & Accuracy & ECE \\ 
			\hline 
			ACNML-SWAGD (ours) & $0.2167\pm0.0041$ & $93.23\pm 0.09$ & $0.0115\pm0.0010$ &
			$0.1130\pm 0.0012$ & $96.38 \pm 0.03$ & $0.0122\pm 0.0006$\\
			ACNML-KFAC (ours) & $0.2329\pm0.0028$ & $93.14\pm0.08$ & $0.0361\pm 0.0016$ & - & - & - \\
			MAP (SWA) & $0.2694\pm0.0056$ & $93.23\pm0.13$  & $0.0430\pm0.0010$ & $0.1128 \pm 0.0014$ & $96.41\pm0.01$ & $0.0099 \pm 0.0004$\\
			SWAGD & $0.2257\pm0.0047$ & $93.31\pm0.04$ & $0.0284\pm0.0002$ &
			$0.1125 \pm 0.0012$ & $96.28\pm 0.04$ & $\mathbf{0.0042} \pm  0.0003$ \\ 
			SWAG & $0.2016 \pm 0.0031$ & $93.60 \pm 0.10$ & $0.0158 \pm 0.0030$ & $ 0.1122 \pm 0.0009$ & $96.32 \pm 0.08$ & $0.0088 \pm 0.0006$  \\ 
			KFAC-Laplace & $0.2236\pm0.0013$ & $92.76\pm0.11$ & $\mathbf{0.0097}\pm 0.0005$ & $0.1197\pm0.0031$ & $96.23\pm0.02$ & $0.0111 \pm 0.0006$\\
			SWA-Dropout & $0.2562\pm0.0025$ & $92.85\pm0.14$ & $0.0380\pm0.0007$ & $0.1111\pm 0.0024$ & $96.36\pm0.09$ &     $0.0107 \pm 0.0008$\\
			SWA-Temp & $0.2481 \pm 0.0245$ & $93.61\pm 0.11$ & $0.0366 \pm 0.0063$ & $\mathbf{0.1064} \pm 0.0004$ & $96.46 \pm 0.04$ & $0.0080 \pm 0.0007$ \\
			SGD & $0.3285\pm0.0139$ & $93.17\pm0.14$ & $0.0483\pm0.0022$ & $0.1294\pm0.0022$ & $96.41 \pm 0.10$ & $0.0166 \pm 0.0007$ \\
			SWA-Ensemble & $\textbf{0.17867}$ & $\textbf{94.36}$ & $0.0148$  & $\textbf{0.1036}$ & $\textbf{96.53}$ & $0.0068$  \\
			\hline
		\end{tabular} 
}
\caption{\textbf{In-distribution comparative results}
We see that for in-distribution performance, ACNML variants perform comparably to other methods, without large separations between most methods.
Results for SWA-Temp and SGD are taken from \citet{Maddox2019ALearning}.}
\label{tab:nll-acc-ece-indist}
\end{table*}

\begin{figure*}
    \centering
\begin{subfigure}[t]{.7\linewidth}
    \includegraphics[width=\linewidth]{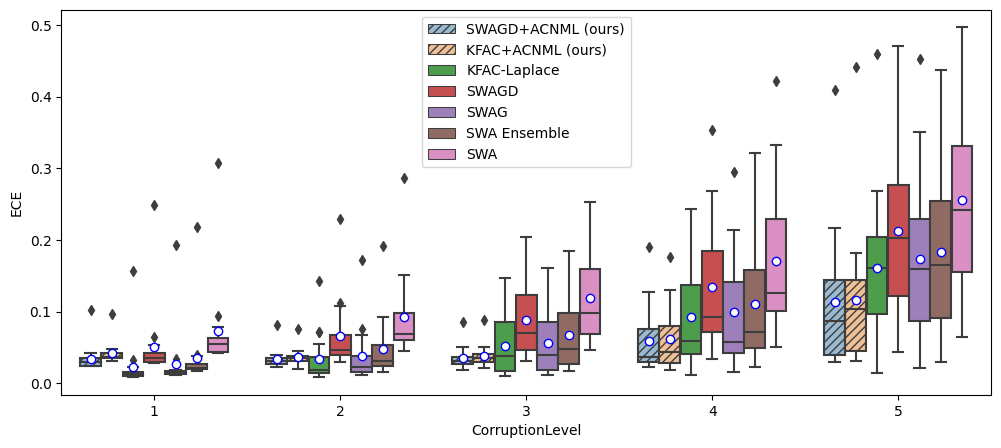}
    \vspace*{-6mm}
    \caption{CIFAR10C VGG16 ECEs (lower is better)}
    \label{fig:cifarc-vgg-ece}
\end{subfigure}%
\\
    \begin{subfigure}[t]{.7\linewidth}
    \includegraphics[width=\linewidth]{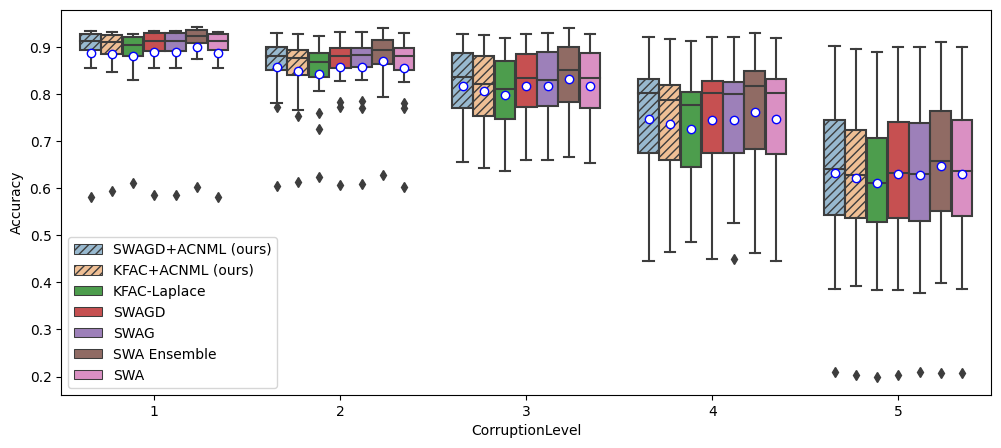}
    \vspace*{-6mm}
    \caption{CIFAR10C VGG16 Accuracies (higher is better)}
    \label{fig:cifarc-vgg-nll}
\end{subfigure}%
\\
\begin{subfigure}[t]{.7\linewidth}
    \includegraphics[width=\linewidth]{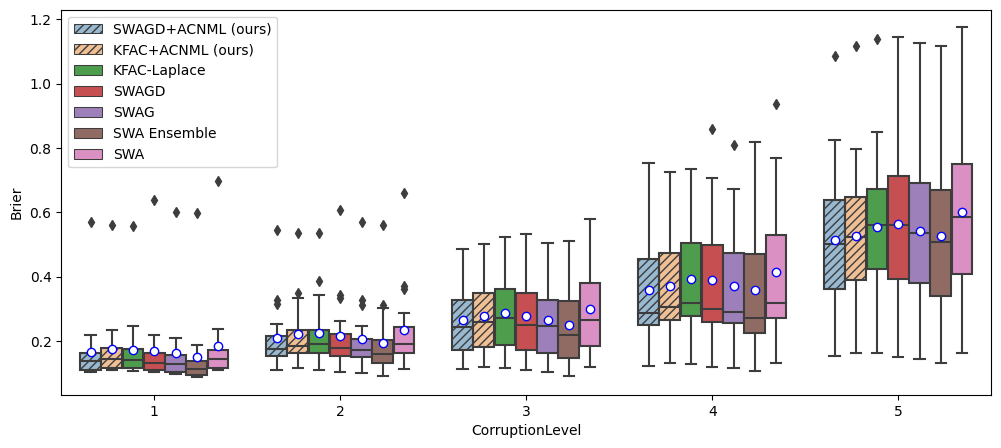}
    \vspace*{-6mm}
    \caption{CIFAR10C VGG16 Brier scores (lower is better)}
    \label{fig:cifarc-vgg-nll}
\end{subfigure}%
\\
\begin{subfigure}[t]{.7\linewidth}
    \includegraphics[width=\linewidth]{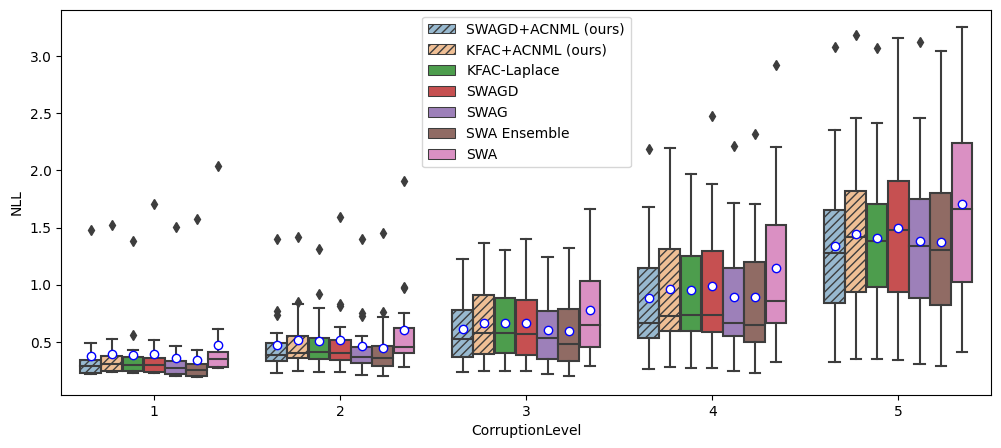}
    \vspace*{-6mm}
    \caption{CIFAR10C VGG16 NLLs (lower is better)}
    \label{fig:cifarc-vgg-nll}
\end{subfigure}%
\vspace*{-2mm}
\caption{
CIFAR10-C performance with the VGG16 architecture. Instantations of our methods are shown in stripes.
Boxplots show quartiles of each statistic over all different corruption types of the given intensity, with the mean indicated by a circle. Both ACNML variants attain significantly better ECE (a) on the more severe corruptions, as the images move further out of distribution.}
\label{fig:cifarc-results-vgg}
\end{figure*}

\begin{figure*}
    \centering
    
\begin{subfigure}[t]{.7\linewidth}
    \includegraphics[width=\linewidth]{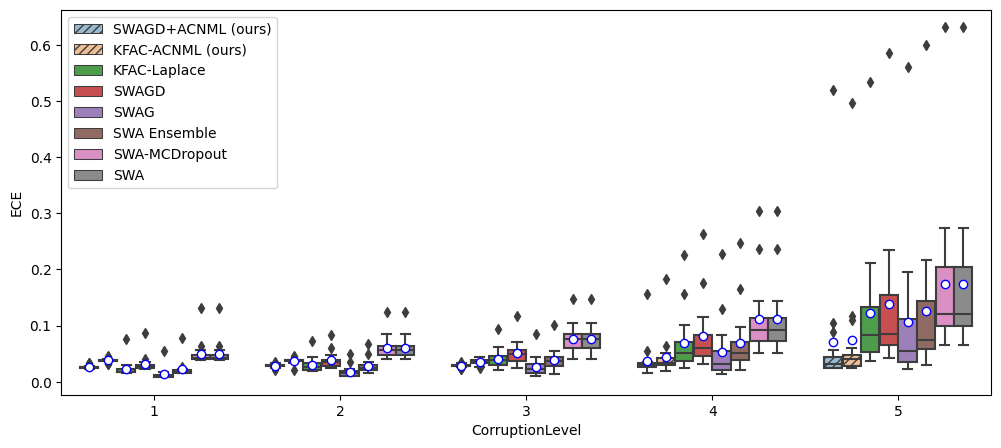}
    \vspace*{-6mm}
    \caption{CIFAR10C VGG16Drop ECEs (lower is better)}
    \label{fig:cifarc-vggdrop-ece}
\end{subfigure}%
\\
    \begin{subfigure}[t]{.7\linewidth}
    \includegraphics[width=\linewidth]{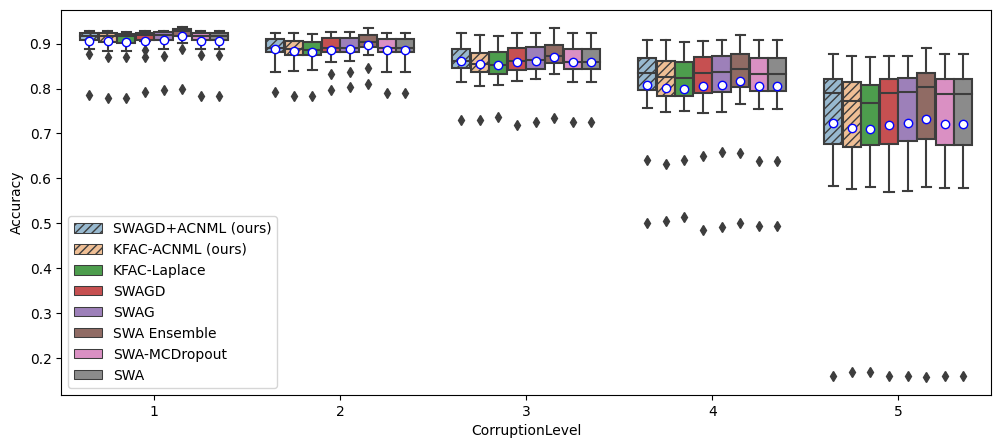}
    \vspace*{-6mm}
    \caption{CIFAR10C VGG16Drop Accuracies (higher is better)}
    \label{fig:cifarc-vggdrop-nll}
\end{subfigure}%
\\  
\begin{subfigure}[t]{.7\linewidth}
    \includegraphics[width=\linewidth]{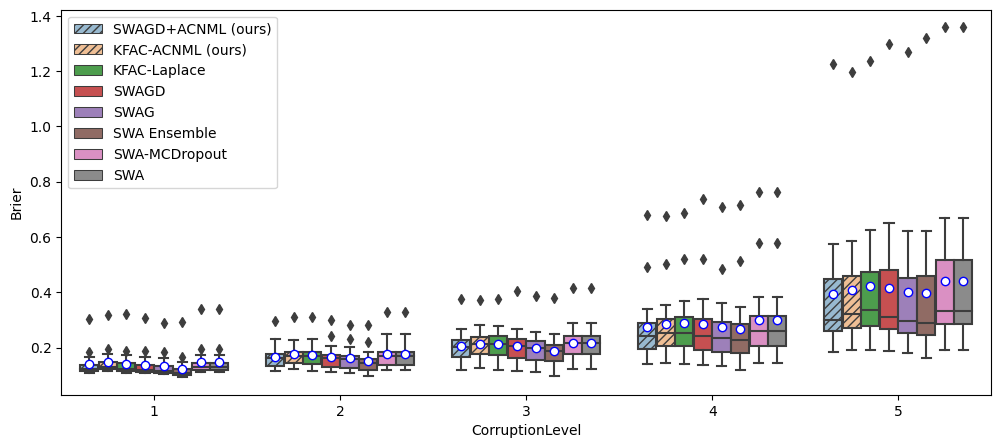}
    \vspace*{-6mm}
    \caption{CIFAR10C VGG16Drop Brier scores (lower is better)}
    \label{fig:cifarc-vggdrop-nll}
\end{subfigure}%
\\
\begin{subfigure}[t]{.7\linewidth}
    \includegraphics[width=\linewidth]{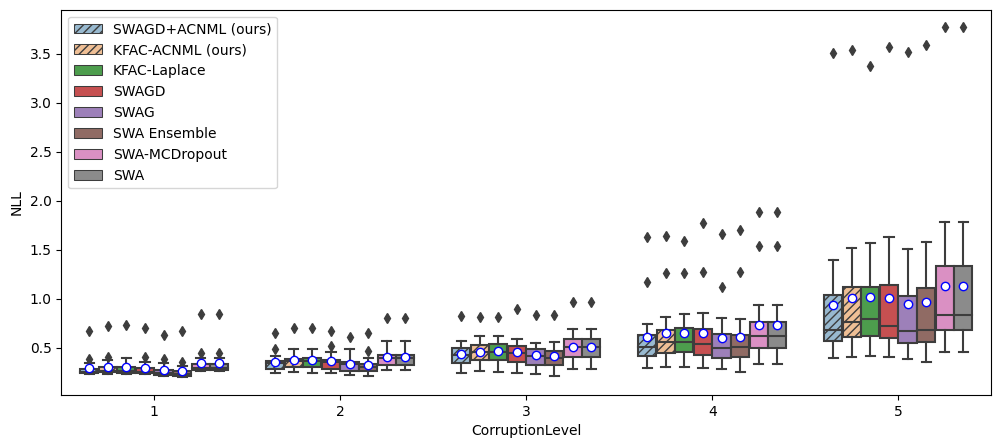}
    \vspace*{-6mm}
    \caption{CIFAR10C VGG16Drop NLLs (lower is better)}
    \label{fig:cifarc-vgg-nll}
\end{subfigure}%
    \vspace*{-2mm}
\caption{
CIFAR10-C performance with the VGG16Drop architecture. Instantations of our methods are shown in stripes.
Boxplots show quartiles of each statistic over all different corruption types of the given intensity, with the mean indicated by a circle. Again, both ACNML variants attain significantly better ECE (a) on the more severe corruptions, as the images move further out of distribution.}
\label{fig:cifarc-results-vgg-drop}
\end{figure*}

\begin{figure*}
    \centering
    \begin{subfigure}[t]{.7\linewidth}
    \includegraphics[width=\linewidth]{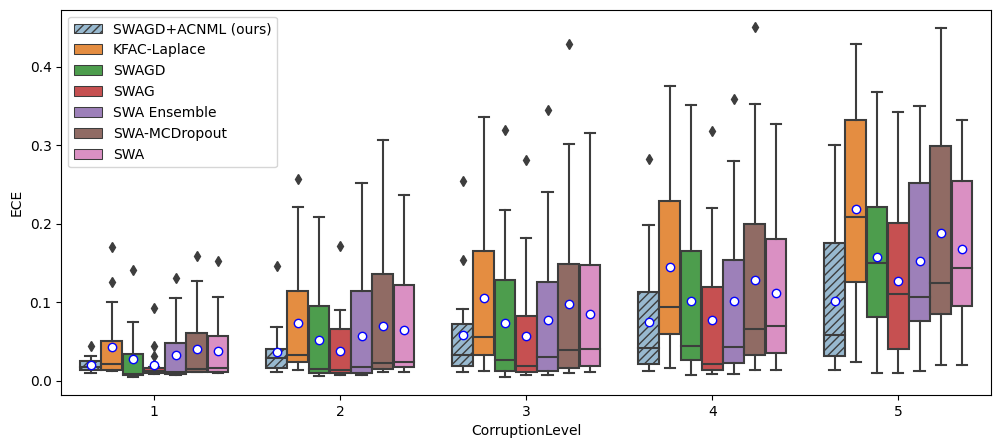}
    \vspace*{-6mm}
    \caption{CIFAR10C WRN28x10 ECEs (lower is better)}
    \label{fig:cifarc-vgg-ece}
    \end{subfigure}%
    \\
    \begin{subfigure}[t]{.7\linewidth}
    \includegraphics[width=\linewidth]{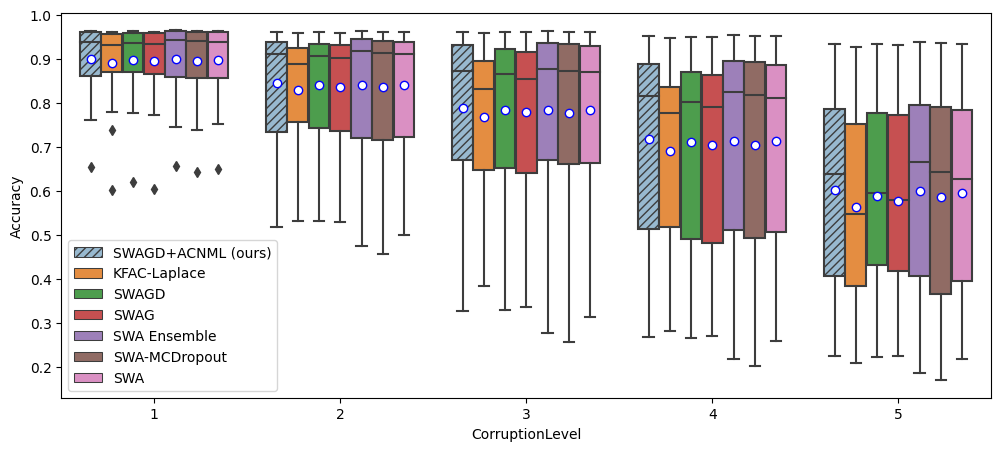}
    \vspace*{-6mm}
    \caption{CIFAR10C WRN28x10 Accuracies (higher is better)}
    \label{fig:cifarc-vgg-nll}
    \end{subfigure}%
    \\
\begin{subfigure}[t]{.7\linewidth}
    \includegraphics[width=\linewidth]{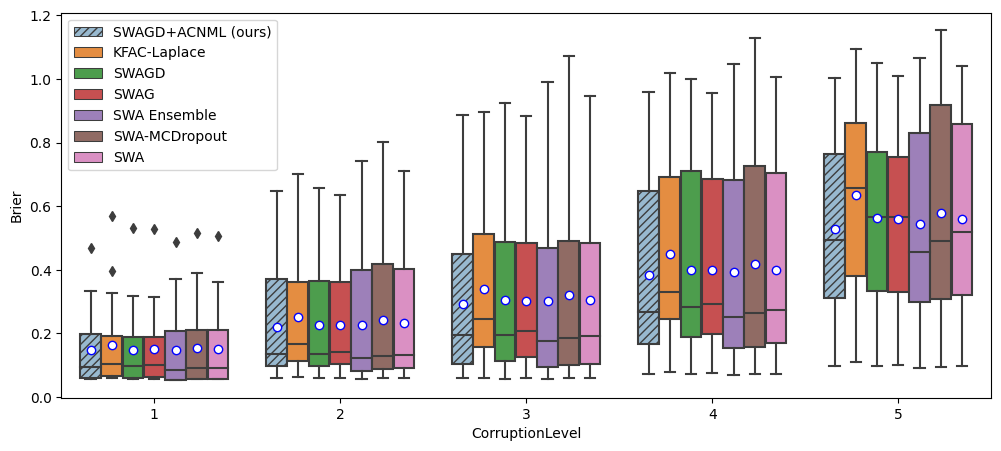}
    \vspace*{-6mm}
    \caption{CIFAR10C WRN28x10 Brier scores (lower is better)}
    \label{fig:cifarc-vgg-nll}
    \end{subfigure}%
    \\
\begin{subfigure}[t]{.7\linewidth}
    \includegraphics[width=\linewidth]{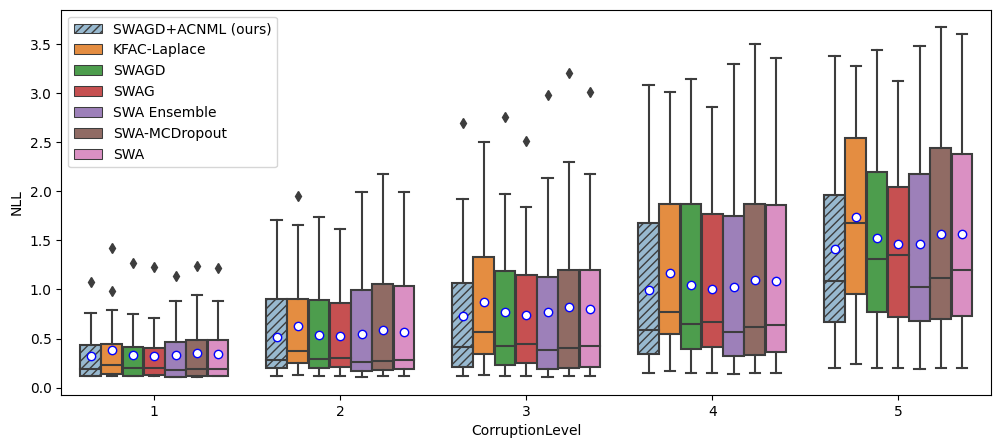}
    \vspace*{-6mm}
        \caption{CIFAR10C WRN28x10 NLLs (lower is better)}
    \label{fig:cifarc-vgg-nll}
    \end{subfigure}%
    \\
    \vspace*{-2mm}
\caption{
CIFAR10-C performance with the WideResNet28x10 architecture. Instantations of our methods are shown in stripes.
Boxplots show quartiles of each statistic over all different corruption types of the given intensity, with the mean indicated by a circle. Again, we see that ACNML attains better ECE values than comparable methods on the heavier corruptions (b).}
\label{fig:cifarc-results-wrn}
\end{figure*}

\newpage 
\section{Comparisons between ACNML and naive CNML on MNIST} \label{appendix:cnml-vs-acnml}
In this section, we include expanded comparisons between ACNML and a naive implementation of CNML from \citet{Bibas2019DeepNetworks} that computes the MLE/MAP $\hat \theta_y$ for each label by appending the query point and label to the dataset and finetuning for $N$ epochs. Both ACNML and naive CNML are initialized from the same MAP solution, with ACNML taking 5 gradient steps on the query point and posterior and naive CNML finetuning with the query point and training set for 5 epochs. For the OOD dataset, instead of computing results for every level of corruption, we instead simply average out over all corruption levels by randomly rotating the test inputs.

This naive implementation of CNML differs slightly from \citet{Bibas2019DeepNetworks} in that we finetune the entire network, while \citet{Bibas2019DeepNetworks} proposed only tuning the last few layers. During the finetuning, we also append the query point and label to every batch in optimization, and downweighting that portion of the loss accordingly to get unbiased gradient estimates. We found this led to more efficient optimization than randomly sampling 

We first examine how closely ACNML and naive CNML's predictions match on the same datapoint. 
To assess this, we compare the CNML normalization terms $\sum_y p_{\hat \theta_y}(y\vert x)$, NLLs, and the confidences of the two methods. 
The CNML normalization term captures how much each procedure was able to adapt to different labels for that input. A higher normalization term for an input means that we were flexible enough to fit multiple different labels well together with the training set (or approximate posterior in the case of ACNML), and typically means a less confident prediction on that input.

In Figures \ref{fig:indist-cnml-vs-acnml-scatters} and \ref{fig:ood-cnml-vs-acnml-scatters}, We show scatter plots over 1000 randomly selected test points (from the in-distribution test set and the randomly rotated OOD images respectively) comparing the CNML normalizers, NLLs, and confidences of ACNML and naive CNML. In each scatter plot, we include a diagonal red line to illustrate where points would lie if predictions of ACNML and naive CNML matched exactly.

We additionally plot reliability diagrams for MNIST experiments in Figure \ref{fig:mnist_reliability_diagrams}, showing that ACNML provides very conservative predictions.

For the in-distribution test set, we see from the CNML normalizer plot that the ACNML adaptation procedure using the approximate posterior is much less constraining than using the training set, resulting in the normalizers being higher for ACNML than naive CNML for almost all inputs. This leads to excess conservatism, with ACNML almost always having lower confidence its predictions. As a result, we see that on many points where naive CNML outputted confident correct answers and achieved close to 0 NLL loss, ACNML still incurs some higher losses due to its less confident predictions.

On the OOD rotated images, we again see that ACNML typically adapts more than CNML as measured by the CNML normalizers, though the difference is much less extreme compared to the in-distribution dataset. In the confidence scatter plot, we again see that ACNML tends to make lower confidence predictions than naive CNML (especially when naive CNML's predictions are confident), and as seen in Figure \ref{fig:mnist-lineplots-expanded}, result in ACNML having better Brier scores, NLL and calibration on the OOD inputs.

\begin{figure}[t]
    \centering
    \begin{subfigure}{.33\linewidth}
        \includegraphics[width=\linewidth]{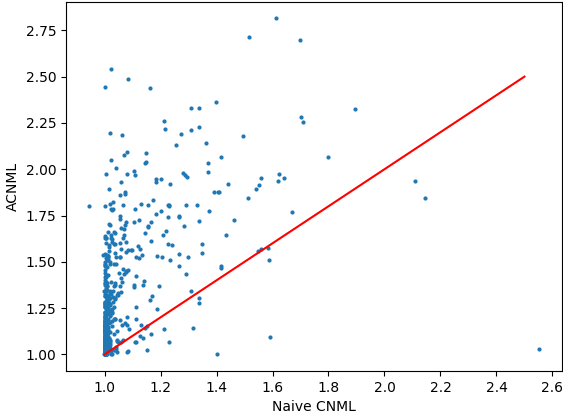}
        \caption{CNML Normalizers $\sum_y p_{\hat \theta_y}(y\vert x)$}
    \end{subfigure} %
    \begin{subfigure}{.33\linewidth}
        \includegraphics[width=\linewidth]{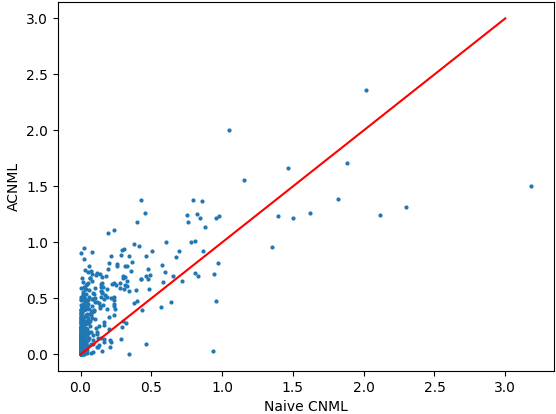}
        \caption{NLLs}
    \end{subfigure}%
    \begin{subfigure}{.33\linewidth}
        \includegraphics[width=\linewidth]{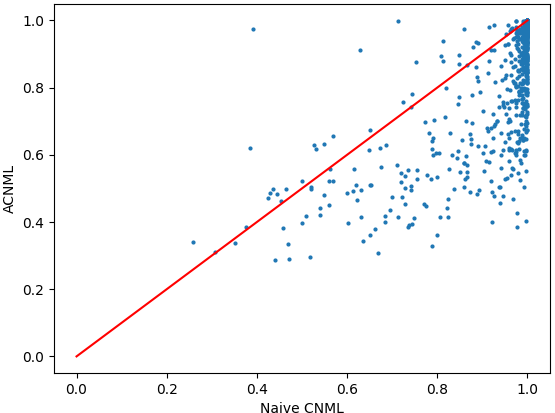}
        \caption{Confidences}
    \end{subfigure}
    \caption{\textbf{In Distribution Comparisons between ACNML and naive CNML.} We plot scatter plots of the values of each statistic for naive CNML (x-axis) vs ACNML (y-axis), with the red line indicating   Looking at the CNML normalizers, we see that the ACNML adaptation procedure using the approximate posterior is much less constraining than using the training set, resulting in the normalizers being higher for ACNML than naive CNML for almost all inputs. This leads to excess conservatism, with ACNML almost always having lower confidence its predictions, and many inputs with close to 0 NLL with naive CNML having higher NLL with ACNML.}
    \label{fig:indist-cnml-vs-acnml-scatters}
\end{figure}
\begin{figure}[t]
    \centering
    \begin{subfigure}{.33\linewidth}
        \includegraphics[width=\linewidth]{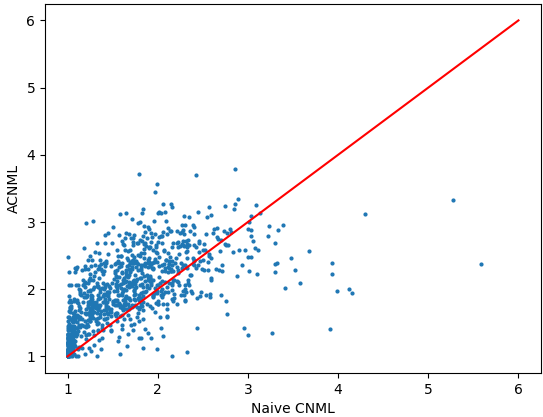}
        \caption{CNML Normalizers $\sum_y p_{\hat \theta_y}(y\vert x)$}
    \end{subfigure}%
    \begin{subfigure}{.33\linewidth}
        \includegraphics[width=\linewidth]{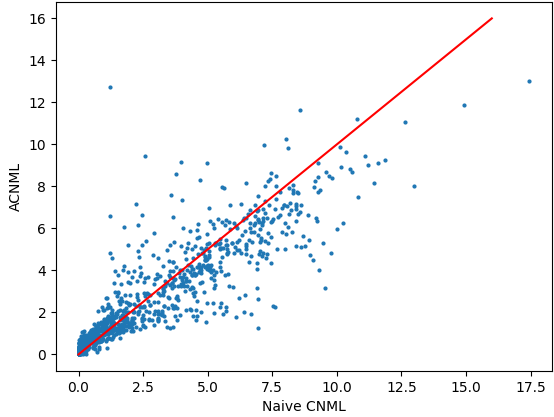}
        \caption{NLLs}
    \end{subfigure}%
    \begin{subfigure}{.33\linewidth}
        \includegraphics[width=\linewidth]{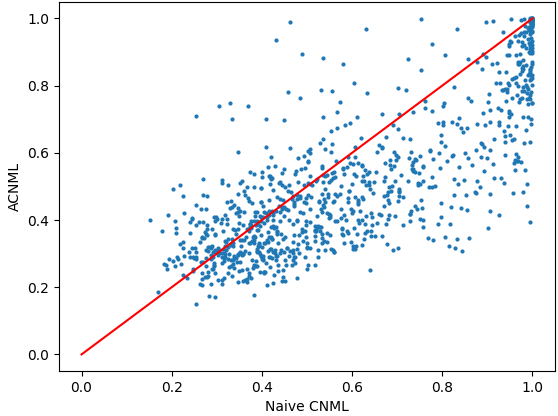}
        \caption{Confidences}
    \end{subfigure}
    \caption{\textbf{OOD Comparisons between ACNML and naive CNML.} We plot scatter plots of the values of each statistic for naive CNML (x-axis) vs ACNML (y-axis). Looking at the CNML normalizers, we again see that the ACNML adaptation procedure using the approximate posterior is less constraining than using the training set, with the normalizers being higher for ACNML than naive CNML for most inputs (though to lesser extent than the in-distribution data). ACNML again outputs more conservative predictions with lower confidence on many inputs, which leads to better NLL and calibration on the OOD dataset, unlike with the in-distribution test set.}
    \label{fig:ood-cnml-vs-acnml-scatters}
\end{figure}

\begin{figure*}
    \centering
        \begin{subfigure}[t]{0.5\textwidth}
            \centering
            \includegraphics[width=\linewidth]{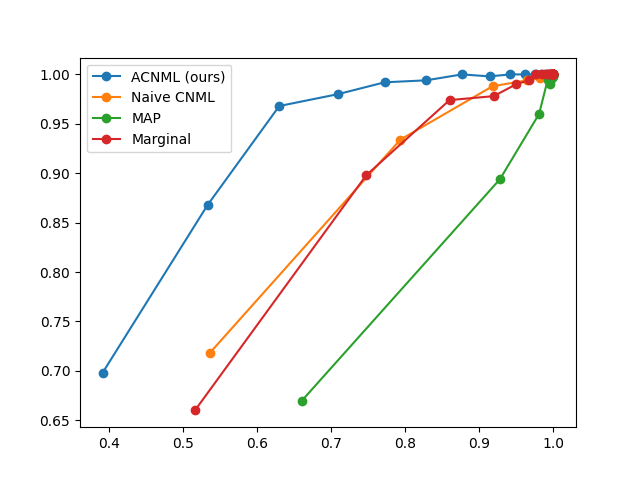}
            \caption{MNIST Test Set}
            \label{fig:reliability:mnist-test}
        \end{subfigure}%
        \begin{subfigure}[t]{.5\textwidth}
            \centering
            \includegraphics[width=\linewidth]{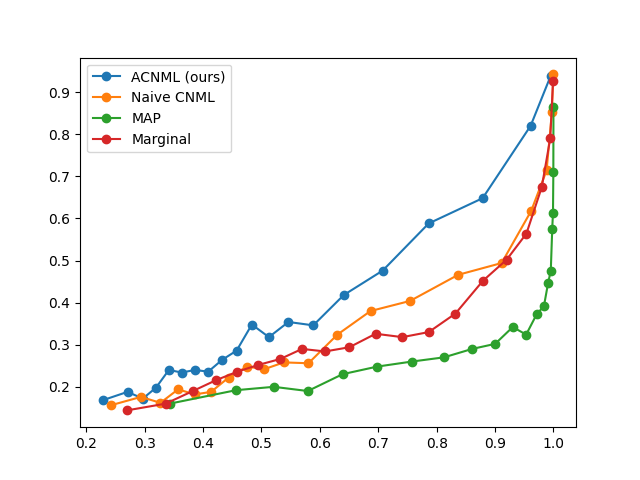}
            \caption{Randomly Rotated MNIST (OOD data) }
            \label{fig:reliability:mnist-ood}
        \end{subfigure}%
        \caption{Reliability diagrams plotting confidence vs. accuracy for Bayes-by-Backprop experiments on the MNIST test set and a randomly rotated MNIST test set (OOD). ACNML's conservative predictions provided better calibrated predictions on the OOD test set.}
    \label{fig:mnist_reliability_diagrams}
\end{figure*}

\textbf{Handling multiple MLEs in CNML}: Strictly speaking, the CNML distribution is not well defined when there exist multiple potential MLEs $\hat \theta_y$ that can output different predictions (prior references to CNML typically assume such MLEs are unique). However, the non-convexity of the objective for deep neural networks means multiple MLEs can exist, and to properly define CNML in this case, we would need to select a particular MLE to use when assigning probabilities in CNML.
In line with the min-max formulation of CNML, we propose to select the MLE $\hat \theta_y$ that maximizes the likelihood $p_{\hat \theta_y}(y\vert x)$ of the query point and proposed label, as this is the choice that maximizes the regret for that particular label over all MLEs.

With our naive CNML instantiation, we observe that during the finetuning for each query point $x$ and label $y$, the predicted probability of that label $p_{\theta}(y\vert x)$ does not monotonically increase over iterations as we might hope (since we initialize $\theta$ to be the MLE of the training set, then update it to maximize likelihood of the training set with the query point and label), but can potentially oscillate substantially throughout the finetuning process. 
We suspect this is due to the stochasticity in the optimization procedure from sampling minibatches of the training data, which causes the trajectory of parameters can potentially visit several different (approximate) local optima that output different predictions on the query point.
While our instantiation of naive CNML simply used the parameter found at the end of 5 epochs, we additionally compare against a variant that explicitly tries to select the MLE that maximizes the likelihood of the proposed label. 
This variant heuristically uses the bset value of $p_{\theta}(y\vert x)$ over all $\theta$ encountered in the last epoch of finetuning. We see in Figure \ref{fig:mnist-lineplots-expanded} that this variant, denoted naive CNML (max), gives more conservative predictions than naive CNML and improves in NLL and calibration on the more OOD rotated datasets. However, it is still not as conservative as ACNML using the Bayes-by-Backprop posterior, and so does not perform as well on the more severe rotations.

\begin{figure*}
 \centering
        \begin{subfigure}[t]{.5\textwidth}
            \centering
            \includegraphics[width=\linewidth]{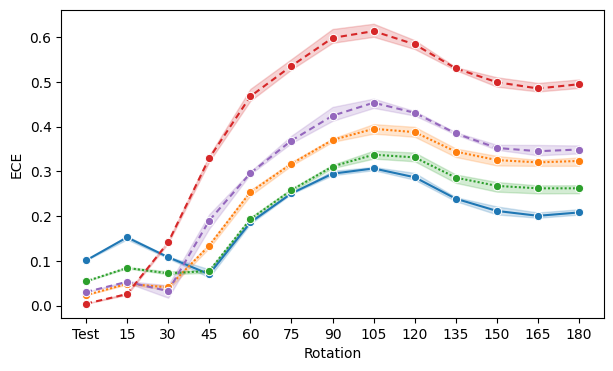}
            \vspace{-6mm}
            \caption{ECE lineplots}
            \label{fig:appendix-mnist-eces}
        \end{subfigure}%
        \begin{subfigure}[t]{.5\textwidth}
            \centering
            \includegraphics[width=\linewidth]{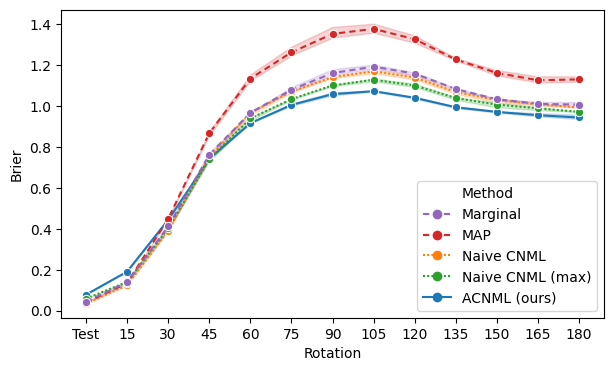}
            \vspace{-6mm}
            \caption{Brier score lineplots}
            \label{fig:appendix-mnist-briers}
        \end{subfigure}%
        \\
        \begin{subfigure}[t]{0.5\textwidth}
            \centering
            \includegraphics[width=\linewidth]{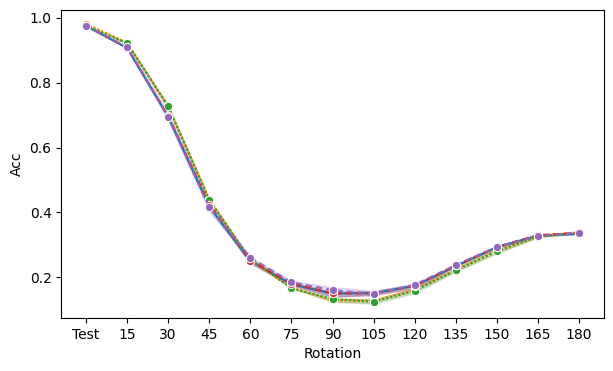}
            \vspace{-6mm}
            \caption{Accuracy lineplots}
            \label{fig:appendix:mnist-accs}
        \end{subfigure}
        \begin{subfigure}[t]{0.5\textwidth}
            \centering
            \includegraphics[width=\linewidth]{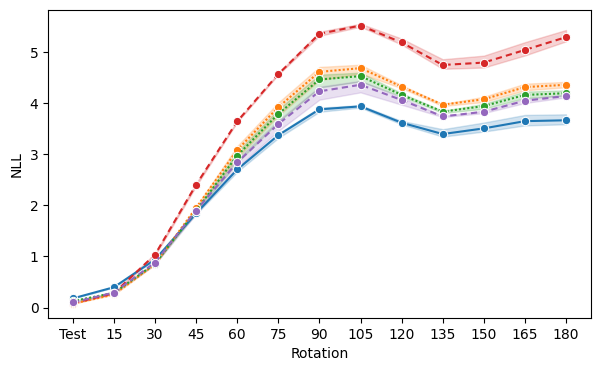}
            \vspace{-6mm}
            \caption{NLL lineplots}
            \label{fig:appendix:mnist-nlls}
        \end{subfigure}%
    
    \vspace{-2mm}
    \caption{\textbf{Expanded MNIST Results}: We include the accuracy and negative-log-likelihood metrics as well as ECE and Brier score. We see that all methods perform similarly in accuracy, and that, and ACNML also has better calibration (ECE), Brier scores, and NLLs on the more OOD datasets compared to other methods. We also additionally compare to the Naive CNML (max) method we designed to handle non-unique maximizers with naive CNML. We see that while the Naive CNML (max) variant outperforms Naive CNML on the more OOD datasets, ACNML is still more conservative, resulting in better calibrated estimates on the more severe rotations. }
    \label{fig:mnist-lineplots-expanded}
\end{figure*}

\section{NMAP and ACNML} \label{appendix:nmap}
NML type methods can be extended with a prior-like regularization term on the selected parameter, resulting in Normalized Maximum a Posteriori (NMAP)\citep{Kakade2006Worst-caseModels}, also referred to as Luckiness NML \citep{Grunwald2004APrinciple}.
For a regularizer given by $\log p(\theta)$, NMAP assigns probabilities according to
\begin{align*}
    p^{\text{NMAP}}(x^n) &\propto p_{\hat \theta(x^n)}(x^n) \\ \hat \theta(x^n) &= \argmax_{\theta} \log p_\theta(x^n) + \log p(\theta).
\end{align*}

Similarly to CNML, there are several variations on NMAP that predict slightly different distributions, but we adopt the one of the same form as our CNML. 
Similarly to how NML was extended to CNML, NMAP can be extended to a conditional version, again with the $\hat \theta$'s being chosen via MAP rather than MLE.
As mentioned in Section 3.1, with a non-uniform prior, ACNML actually approximates a version of conditional NMAP, with the Bayesian prior term on the parameters corresponding to the additional regularizer.

We also note that with the calculations in section 3.1, CNML can be viewed as performing NMAP on a single new test point, with a regularizer corresponding to the posterior likelihood from the training set.
In this perspective, ACNML approximates CNML by using an approximation to that training set regularizer.

\section{Details of Analysis in Section 3.2} \label{appendix:ij-expanded-details}
\subsection{Bounding Error in Parameter Estimation}
Here we state the primary theorem of \citet{giordano2019swiss} along with the necessary definitions and assumptions.

Here, we attempt to estimate an unknown parameter $\theta \in \Omega_\theta \subseteq \R^D$ where $\Omega_{\theta}$ is compact.
Suppose we have a dataset $N$ datapoints and a weight vector $w_1, \ldots, w_N$. 
Let $g_i(\theta)$ denote the gradient of the loss at datapoint $i$ evaluated at $\theta$, and $h_i(\theta)$ the Hessian.
We can then define
\begin{align}
    G(\theta, w) &= \frac{1}{N} \sum_{i=1}^N w_i g_i(\theta) \\
    H(\theta, w) &= \frac{1}{N} \sum_{i=1}^N w_i h_i(\theta).
\end{align}
The MLE $\hat \theta(w)$ for the dataset weighted by $w$ is given by solving for $G(\hat \theta(w), w) = 0$.
Let $1_w$ denote the vector of weights consisting of all 1s. 
We define $\hat \theta_1$ to be the MLE for the whole unweighted dataset, which is equivalent to evaluating $\hat \theta(1_w)$ and also define the corresponding Hessian $H_{1} = H(\hat \theta_1, 1_w)$.
We now wish to estimate $\hat \theta(w)$ using a first order approximation around $\hat \theta_1$ given by 
\newcommand{\ijestimate}{\hat \theta_{\text{IJ}}}
\begin{align}
    \ijestimate(w) = \hat \theta_1 - H_1^{-1} G(\hat \theta_1, \Delta w),
\end{align}
where we define $\Delta_w = w - 1_w$. 
The theorem will proceed to bound $\norm{\hat \theta(w) - \ijestimate}_2$ for suitable weights $w$.

Now we further define $g(\theta) \in \R^{N\times D}$ to be the concatenation of all $g_i(\theta)$s and similarly for $h(\theta) \in \R^{N\times D \times D}.$
We let $\norm{g(\theta)}_p$ and $\norm{h(\theta)}_p$ to refer to the $p$-norms when treating those as vector quantities.

\textbf{Assumption 1} (Smoothness): For all $\theta\in \Omega_\theta$ each $g_n(\theta)$ is continuously differentiable.

\textbf{Assumption 2} (Non-degeneracy): For all $\theta \in \Omega_\theta$, $H(\theta, 1_w)$ is nonsingular and 
\begin{align}
    \sup_{\theta \in \Omega_\theta}\norm{H(\theta, 1_w)^{-1}}_{op} \leq C_{op} \leq \infty.
\end{align}

\textbf{Assumption 3} (Bounded averages): There exist finite constants $C_g$ and $C_h$ such that $\sup_{\theta \in \Omega_{\theta}} \frac{1}{\sqrt{N}}\norm{g(\theta)}_2 \leq C_g$ and $\sup_{\theta \in \Omega_{\theta}} \frac{1}{\sqrt{N}}\norm{h(\theta)}_2 \leq C_h.$

\textbf{Assumption 4} (Local Smoothness): There exists a $\Delta_\theta > 0$ and a finite constant $L_h$ such that $\norm{\theta - \hat \theta_1}_2 \leq \Delta_{\theta}$ implies $\frac{\norm{h(\theta) - h(\hat \theta_1)}_2}{\sqrt N} \leq L_h \norm{\theta - \hat \theta_1}_2.$

\textbf{Assumption 5} (Bounded weight averages). $\frac{1}{\sqrt{N}}\norm{w}_2$ is uniformly bounded for all $w \in W$ by a finite constant $C_w$.

We note that assumption 2 is equivalent to $H_1$ being strongly positive definite. Assumption 5 is not relevant for our use cases, but is stated for completeness.

\textbf{Condition 1} (Set Complexity): There exists a $\delta \geq 0$ and corresponding set $W_\delta \subseteq W$ such that
\begin{align}
    \max_{w\in W_\delta} \sup_{\theta \in \Omega_\theta} \norm{\frac{1}{N}\sum_{i=1}^N (w_i - 1)g_i(\theta)}_1 \leq \delta. \\
    \max_{w\in W_\delta} \sup_{\theta \in \Omega_{\theta}} \norm{\frac{1}{N} \sum_{i=1}^N (w_i - 1) h_i(\theta)}_1 \leq \delta.
\end{align}
Condition 1 essentially describes the set of weight vectors for which $\ijestimate$ will be an accurate approximation within order $\delta$.

\textbf{Definition 1}: Given assumptions 1-5, define 
\begin{align}
    C_{\text{IJ}} &= 1 + DC_wL_h C_{op} \\
    \Delta_{\delta} &= \min\{\Delta_{\theta} C_{op}^{-1}, \frac{1}{n}C_{\text{IJ}}^{-1}C_{op}^{-1}\}.
\end{align}

We now state the main theorem of \citet{giordano2019swiss}.

\textbf{Theorem} (Error Bound for the approximation).
Under assumptions 1-5 and condition 1, 
\begin{align}
    \delta \leq \Delta_\delta \Rightarrow \max_{w\in W_\delta} \norm{\ijestimate(w) - \hat \theta(w)}_2 \leq 2 C_{op}^2 C_{\text{IJ}} \delta^2.
\end{align}

We can now apply the above theorem to provide error bounds for a setting where we have a training set of $n$ datapoints and wish to consider the MLE after adding a new datapoint $z$. 
The issue is that the theorem as stated bounds the error of the approximation when the approximation is centered around the uniform weighting over all the datapoints, which would be appropriate for considering the impact of \textit{removing} datapoints from the dataset.

To apply the theorem to bound the effects of \textit{adding} a datapoint, we have to do some slight manipulation.
We apply the previous theorem with $N= n+2$, where $g_i(\theta)$ correspond to the gradients of training data point $i$ for $i$ in $(1,\ldots, n)$, $g_{n+1} = - \grad \log p_{\theta}(z)$, and $g_{n+2} = \grad \log p_{\theta}(z)$, and similarly for the Hessians $h_i(\theta)$.
We have thus added the query point to the dataset, as well as another fake point that serves to cancel out the contribution of the query point under a uniform weighting, so $G(\theta, 1_w)$ and $H(\theta, 1_w)$ are the mean gradients and Hessians for just the training set.
Now supposing assumptions 1-5 are met for this problem, then we need to check condition $1$ for the particular $W_\delta$ that contains the vector $\bar w$ of all 1s, except for a 2 in the last entry.
We can then find the smallest $\delta$ that satisfies 
\begin{align}
    \sup_{\theta \in \Omega_{\theta}} \norm{\frac{1}{N+2} g_{n+2}(\theta)}_1 \leq \delta \\
    \sup_{\theta \in \Omega_{\theta}} \norm{\frac{1}{N+2} h_{n+2}(\theta)}_1 \leq \delta,
\end{align}
and so long as $\delta \leq \Delta_\delta$, applying the theorem bounds $\norm{\ijestimate(\bar w) - \hat \theta(\bar w)}_2$.

\textbf{Commentary}: The above theorem gives explicit conditions for the accuracy of the approximation that we can verify for a particular training set and query point. 
Under assumptions that we have some limiting procedure for growing the training set such that the constants defined hold uniformly, we can extend this to an asymptotic statement to explicitly say that the approximation error decays as $O(n^{-2})$.

\subsection{Bounding error in the resulting CNML distribution}
We now provide the proof for Proposition \ref{cnml-logit-accuracy-bound}, which we restate here. For notational simplicity, we ignore any dependence on the input $x$, which we consider fixed.
\begin{prop}[3.2]
Suppose $z \in \mathcal Z$ with $\abs{\mathcal Z} = k$ (for example classification with $k$ classes).
Let $\hat \theta_z$ be the exact MLE after appending $z$ to the training set, and let $\tilde \theta_z$ be an approximate MLE with $\norm{\hat \theta_z - \tilde \theta_z} \leq \delta$  for all $z$. 
Further suppose $\log p_{\theta}(z)$ is $L$-Lipschitz in $\theta$.

Denote the exact CNML distribution $p_{\text{CNML}}(z) \propto p_{\hat \theta_z}(z)$ and an approximate CNML distribution $p_{\text{ACNML}}(z) \propto p_{\tilde \theta_z}(z)$. Then, we have the bound
\begin{align}
    \sup_{z} \abs{\log p_{\text{CNML}}(z) - \log p_{\text{ACNML}}(z)} \leq 2L\delta.
\end{align}
\end{prop}
\begin{proof}

The assumed bound $\norm{\hat \theta_z - \tilde \theta_z}_2 \leq \delta$ combined with $L$-Lipschitzness implies a bound on differences of logits of each class
\begin{align} \label{eq:logit_bound_acnml}
    \abs{\log p_{\hat \theta_z}(z) - \log p_{\hat \theta_z}(z)} \leq L\delta.
\end{align}

We note that the log probabilities of the exact CNML distribution $p_{\text{CNML}}$ ($p_{\text{ACNML}}$ is given by a similar expression using $\tilde \theta_z$ instead of $\hat \theta_z$) is given by
\begin{align}
    \log p_{\text{CNML}}(z) &=  \log p_{\hat \theta_z}(z) - \log \sum_{z' \in \mathcal Z} p_{\hat \theta_{z'}}(z').
\end{align}
For any $z \in \mathcal Z$, we can then expand, apply the triangle inequality and then Equation $\ref{eq:logit_bound_acnml}$ to obtain
\begin{align}
    &\abs{\log p_{\text{CNML}}(z) - \log p_{\text{ACNML}}(z)} \nonumber \\ 
    &= \lvert \log p_{\hat \theta_z}(z) - \nonumber  \log p_{\tilde \theta_z}(z) \nonumber \\ & - \log \sum_{z'\in \mathcal Z} p_{\hat \theta_{z'}}(z') + \log \sum_{z' \in \mathcal Z} p_{\tilde \theta_{z'}}(z') \rvert \\
    &\leq \abs{\log p_{\hat \theta_z}(z) - \log p_{\tilde \theta_z}(z)} \nonumber \\ &+ \abs{\log \sum_{z' \in \mathcal Z} p_{\hat \theta_{z'}}(z') - \log \sum_{z' \in \mathcal Z} p_{\tilde \theta_{z'}}(z')} \\
    &\leq L\delta + \abs{\log \sum_{z' \in \mathcal Z} p_{\hat \theta_{z'}}(z') - \log \sum_{z'\in \mathcal Z} p_{\tilde \theta_{z'}}(z')}.
\end{align}

We now bound the difference between the log-normalizers $\abs{\log \sum_{z'} p_{\hat \theta_{z'}}(z') - \log \sum_{z'} p_{\tilde \theta_{z'}}(z')}$.

We first let $p_{\min}(z) = \min\{p_{\hat \theta_z}(z), p_{\tilde \theta_z}(z)\}$ and $p_{\max}(z) = \max\{p_{\hat \theta_z}(z), p_{\tilde \theta_z}(z)\}$, and note that Equation \ref{eq:logit_bound_acnml} implies $\log p_{\max}(z) \leq \log p_{\min}(z) + L\delta$ for all $z$.  We then bound the difference in log-normalizers
\begin{align}
    & \abs{\log \sum_{z'\in \mathcal Z} p_{\hat \theta_{z'}}(z') - \log \sum_{z'\in \mathcal Z} p_{\tilde \theta_{z'}}(z')} \nonumber\\ &\leq \log \sum_{z' \in \mathcal Z} p_{\max} (z') - \log \sum_{z'\in \mathcal Z} p_{\min}(z')  \\ 
    &= \log \frac{\sum_{z'\in \mathcal Z} p_{\max}(z')}{\sum_{z'\in \mathcal Z} p_{\min}(z')} \\
    &= \log \frac{\sum_{z'\in \mathcal Z} \exp(\log p_{\max}(z'))}{\sum_{z'\in \mathcal Z} p_{\min}(z')} \\
    &\leq \log \frac{\sum_{z' \in \mathcal Z} \exp(\log p_{\min}(z') + L\delta)}{\sum_{z'\in \mathcal Z} p_{\min}(z')} \\
    &= \log \frac{ \exp(L\delta) \sum_{z' \in \mathcal Z} p_{\min}(z')}{\sum_{z' \in \mathcal Z} p_{\min}(z')} \\
    &= L\delta.
\end{align}
Plugging back into Equation 37, we have the following bound for all $z\in \mathcal Z$   
\begin{align}
\abs{\log p_{\text{CNML}}(z) - \log p_{\text{ACNML}}(z)} &\leq 2L\delta.
\end{align}
\end{proof}

\end{document}